\documentclass{article}

\usepackage{microtype}
\usepackage{booktabs} 
\usepackage{dsfont}
\usepackage{mathrsfs}
\usepackage{graphicx, subcaption}
\usepackage{hyperref}
\usepackage{url}
\usepackage{amsthm, thm-restate}
\usepackage{amsmath}
\usepackage[capitalise]{cleveref}
\usepackage{nccmath}
\usepackage{mathtools}
\usepackage[acronym,nowarn,section,nogroupskip,nonumberlist]{glossaries}
\usepackage[breakable]{tcolorbox}
\usepackage{algorithmic}
\usepackage{cancel}
\usepackage{bbm}
\usepackage{natbib}
\usepackage[normalem]{ulem}
\usepackage{thmtools}
\usepackage{amsmath,amsfonts,bm}
\global\long\def\spp{\quad}%
\global\long\def\dt{\,\mathrm{d}t}%
\global\long\def\X{\mathcal{X}}%
\global\long\def\dx{\,\mathrm{d}x}%
\global\long\def\dd{\,\mathrm{d}}%

\newcommand{\lagr}{s}
\newcommand{\hclagr}{\Phi}

\newcommand{\ifhl}[1]{#1}

\newacronym{WFR}{WFR}
{Wasserstein Fisher-Rao}
\newacronym{OT}{OT}
{optimal transport}
\newacronym{SB}{SB}
{Schrödinger Bridge}
\newacronym{AM}{AM}
{Action Matching}

\newcommand{\marginalcouple}{\Gamma(\mu_0, \mu_1)}
\newcommand{\mmarginalcouple}{\Gamma(\{\mu_{t_i}\})}
\newcommand{\marginalcoupleof}[1]{\Gamma(\{ #1 \})}
\newcommand{\hamilgs}{H}
\theoremstyle{definition}
\newtheorem{example}{Example}[section]

\crefname{example}{Ex.}{example}
\crefname{definition}{Def.}{def}
\crefname{section}{Sec.}{sec}
\crefname{appendix}{App.}{app}
\crefname{algorithm}{Alg.}{algorithm}
\crefname{theorem}{Thm.}{theorem}

\theoremstyle{plain}
\newtheorem{definition}{Definition}[section]

\newcommand{\x}{x_0}
\newcommand{\y}{x_1}
\newcommand{\st}{\quad \text{s.t.}\,\,}

\newcommand{\wmanifold}{\cP_2(\cX)}
\newcommand{\wfrmanifold}{\cM(\cX)}
\newcommand{\xmetricnorm}[1]{\| #1\|}
\newcommand{\qm}{\rho}
\newcommand{\vel}{v}

\newcommand{\tmetricat}[2]{\langle #1\rangle_{T_{#2}}}
\newcommand{\cometricat}[2]{\langle #1\rangle_{T^*_{#2}}}

\newcommand{\wfrmetricat}[2]{\langle #1\rangle^{WFR_{\lambda}}_{T_{#2}}}
\newcommand{\wfrcometricat}[2]{\langle #1\rangle^{WFR_{\lambda}}_{T^*_{#2}}}
\newcommand{\wmetricat}[2]{\langle #1\rangle^{W_2}_{T_{#2}}}
\newcommand{\wnormat}[2]{\left\| #1\right\|^2_{T^{W_2}_{#2}}} 
\newcommand{\wcometricat}[2]{\langle #1\rangle^{W_2}_{T^*_{#2}}}

\newcommand{\accel}{p}

\newcommand{\lagrv}{}
\newcommand{\lagrsbat}[1]{\frac{\sigma_{#1}^4}{8}}
\newcommand{\sqrttwolagrsbat}[1]{\frac{\sigma_{#1}^2}{2}}

\tcbset{mybox/.style={colback=red!5, width =\textwidth, boxrule=0.0pt, top=5pt,bottom=5pt,left=2pt, right=2pt},
  before upper=\setlength{\parindent}
  {0em}\everypar{{\setbox0\lastbox}\everypar{}}
}
\newtcolorbox{mybox}[1][]{breakable, mybox,#1}

\def\eqref#1{equation~\ref{#1}}

\def\1{\bm{1}}

\DeclareMathAlphabet{\mathsfit}{\encodingdefault}{\sfdefault}{m}{sl}
\SetMathAlphabet{\mathsfit}{bold}{\encodingdefault}{\sfdefault}{bx}{n}

\newcommand{\cA}{{\mathcal{A}}}

\newcommand{\cC}{{\mathcal{C}}}

\newcommand{\cF}{{\mathcal{F}}}

\newcommand{\cH}{{\mathcal{H}}}

\newcommand{\cK}{{\mathcal{K}}}
\newcommand{\cL}{{\mathcal{L}}}
\newcommand{\cM}{{\mathcal{M}}}

\newcommand{\cP}{{\mathcal{P}}}

\newcommand{\cS}{{\mathcal{S}}}
\newcommand{\cT}{{\mathcal{T}}}
\newcommand{\cU}{{\mathcal{U}}}

\newcommand{\cX}{{\mathcal{X}}}

\newcommand{\bbR}{{\mathbb{R}}}

\global\long\def\divp#1{\nabla\cdot\left(#1\right)}%

\let\originalleft\left
\let\originalright\right
\renewcommand{\left}{\mathopen{}\mathclose\bgroup\originalleft}
\renewcommand{\right}{\aftergroup\egroup\originalright}
\global\long\def\X{\mathcal{X}}

\newcommand{\cd}{\cdot}

\global\long\def\inner#1#2{\left\langle #1, #2\right\rangle}

\definecolor{antiquefuchsia}{rgb}{0.57, 0.36, 0.51}
\definecolor{amethyst}{rgb}{0.6, 0.4, 0.8}

\newcommand{\deriv}[2]{\frac{\partial #1}{\partial #2}}

\newcommand{\var}{{\rm I\kern-.3em D}}

\newcommand{\cond}{\,|\,}

\newtheorem{theorem}{Theorem}
\newtheorem{proposition}[theorem]{Proposition}
\DeclareMathSymbol{\shortminus}{\mathbin}{AMSa}{"39}

\usepackage{color}

\definecolor{bleudefrance}{rgb}{0.19, 0.55, 0.91}
\definecolor{azure(colorwheel)}{rgb}{0.0, 0.5, 1.0}

\usepackage{hyperref}
\glsdisablehyper

\usepackage[accepted]{icml2024}

\usepackage{amsmath}
\usepackage{amssymb}
\usepackage{mathtools}
\usepackage{amsthm}

\newcommand{\vheader}{}

\theoremstyle{plain}

\usepackage[textsize=tiny]{todonotes}

\icmltitlerunning{A Computational Framework for Solving Wasserstein Lagrangian Flows}

\begin{document}

\twocolumn[
\icmltitle{A Computational Framework for Solving Wasserstein Lagrangian Flows}



\icmlsetsymbol{equal}{*}

\begin{icmlauthorlist}
\icmlauthor{Kirill Neklyudov}{equal,udem,mila}
\icmlauthor{Rob Brekelmans}{equal,vector}
\icmlauthor{Alexander Tong}{udem,mila}
\icmlauthor{Lazar Atanackovic}{vector,uoft} \\ 
\hspace*{.4cm} 
\icmlauthor{Qiang Liu}{utaustin}\hspace*{.2cm}
\icmlauthor{Alireza Makhzani}{vector,uoft}
\end{icmlauthorlist}

\icmlaffiliation{udem}{Université de Montréal}
\icmlaffiliation{mila}{Mila - Quebec AI Institute}
\icmlaffiliation{vector}{Vector Institute}
\icmlaffiliation{uoft}{University of Toronto}
\icmlaffiliation{utaustin}{University of Texas at Austin}

\icmlcorrespondingauthor{}{k.necludov@gmail.com}
\icmlcorrespondingauthor{}{\{rob.brekelmans, makhzani\}@vectorinstitute.ai}

\icmlkeywords{Machine Learning, ICML}

\vskip 0.3in
]



\printAffiliationsAndNotice{\icmlEqualContribution} 

\begin{abstract}
The dynamical formulation of the optimal transport can be extended through various choices of the underlying geometry (\textit{kinetic energy}), and the regularization of density paths (\textit{potential energy}). These combinations yield different variational problems (\textit{Lagrangians}), encompassing many variations of the optimal transport problem such as the Schrödinger bridge, unbalanced optimal transport, and optimal transport with physical constraints, among others. In general, the optimal density path is unknown, and solving these variational problems can be computationally challenging. We propose a novel deep learning based framework approaching all of these problems from a unified perspective. Leveraging the dual formulation of the Lagrangians, our method does not require simulating or backpropagating through the trajectories of the learned dynamics, and does not need access to optimal couplings. We showcase the versatility of the proposed framework by outperforming previous approaches for the single-cell trajectory inference, where incorporating prior knowledge into the dynamics is crucial for correct predictions.
\end{abstract}

\vspace{-.5cm}


\vspace*{-.12cm}
\section{Introduction}\label{sec:intro}
The problem of \textit{trajectory inference}, or recovering the population dynamics of a system from samples of its temporal marginal distributions, is a challenging problem arising throughout the natural sciences \citep{hashimoto2016learning, lavenant2021towards}.   
A particularly important application is analysis of single-cell RNA-sequencing data \citep{schiebinger2019optimal, schiebinger2021reconstructing, saelens2019comparison}, which provides a heterogeneous snapshot of a cell population at a high resolution, allowing high-throughput observation over tens of thousands of genes \citep{macosko2015highly}.   However, since the measurement process ultimately leads to cell death, we can only observe temporal changes of the \textit{marginal} or \textit{population} distributions of cells as they undergo 
treatment, differentiation, or developmental processes of interest.
To understand these processes and make future predictions, we are interested in both (i) interpolating the evolution of marginal cell distributions between observed timepoints and (ii) modeling the full trajectories at the individual cell level.


\begin{figure*}[t]
\begin{minipage}{.24\textwidth}
\centering
\includegraphics[width=\textwidth]{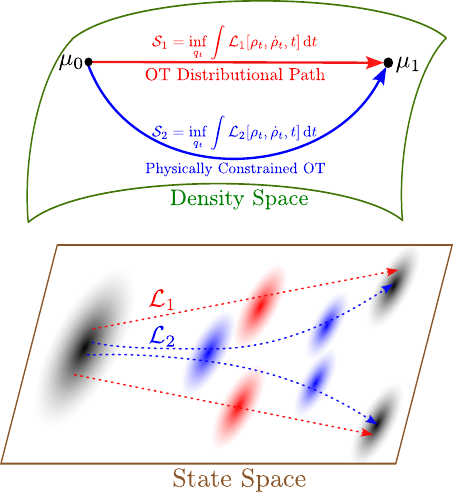}
\end{minipage}
\begin{minipage}{.74\textwidth}
\centering
  \begin{subfigure}[t]{.24\textwidth}
    \centering\includegraphics[width=\textwidth]{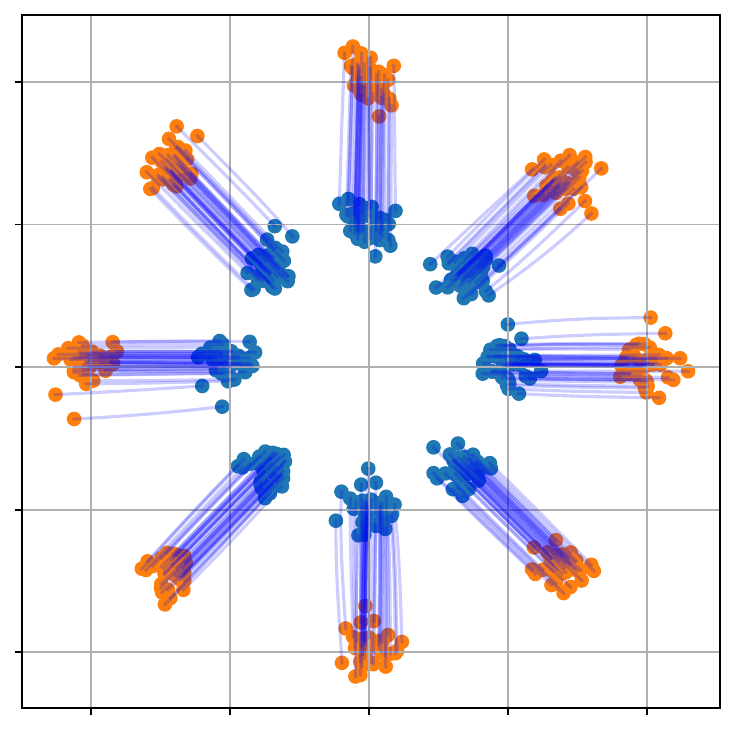}
    \vspace*{.2cm} \caption{Standard \textsc{ot}}
    \label{fig:examples_a}
  \end{subfigure}
    \begin{subfigure}[t]{.24\textwidth}
    \centering\includegraphics[width=\textwidth]{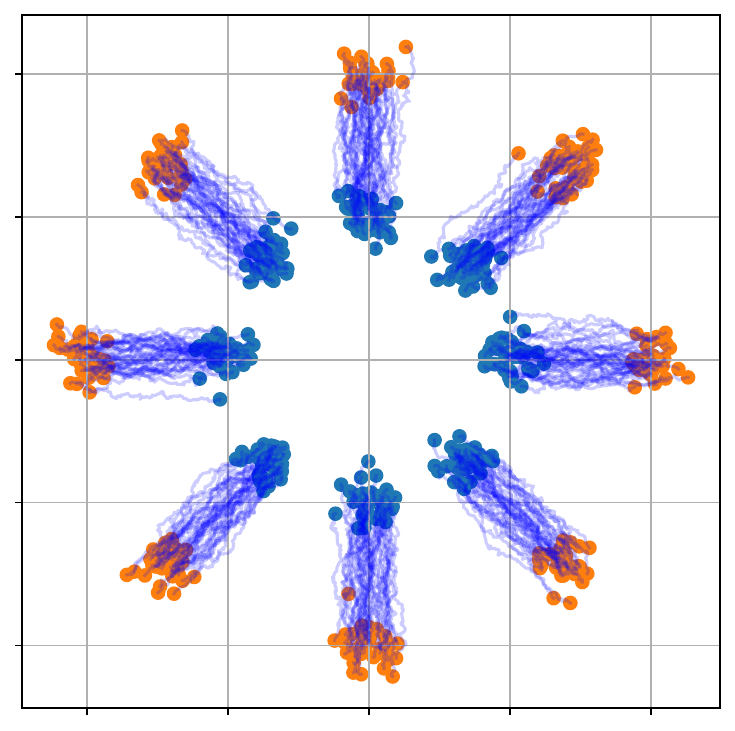}
    \vspace*{.2cm} 
    \caption{\textsc{sb}/Entropic \textsc{ot}}
  \end{subfigure}
    \begin{subfigure}[t]{.24\textwidth}
    \centering\includegraphics[width=\textwidth]{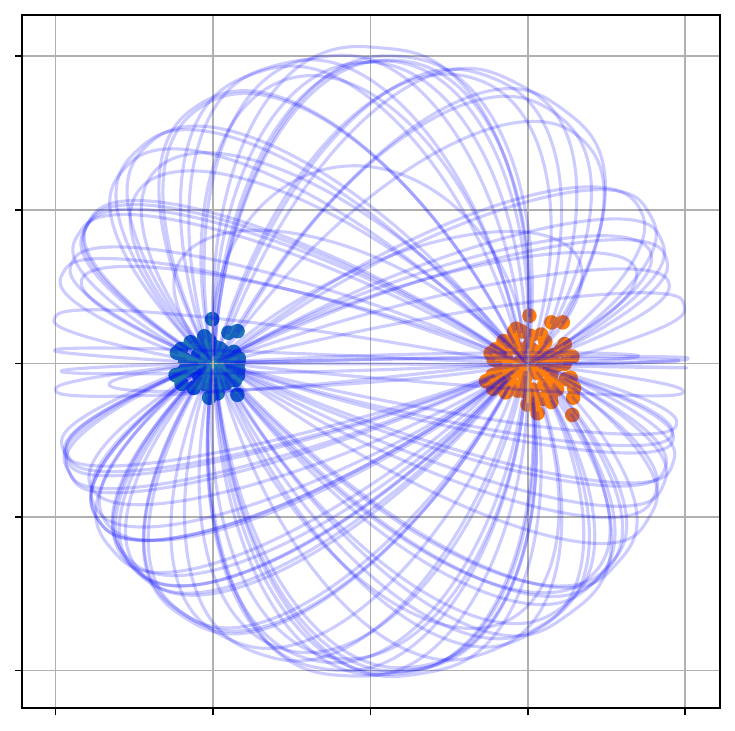}
    \vspace*{.2cm} 
    \caption{Physical \textsc{ot}}
  \end{subfigure}
      \begin{subfigure}[t]{.24\textwidth}
    \centering\includegraphics[width=\textwidth]{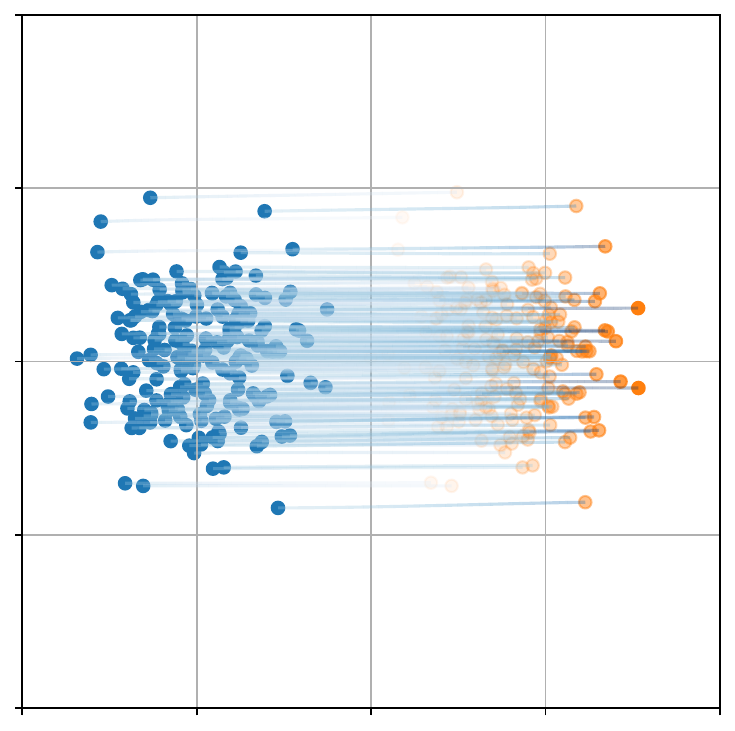}
    \vspace*{.2cm} 
    \caption{Unbalanced \textsc{ot}}
\end{subfigure}
\end{minipage}
\vspace*{-.1cm}
\caption{Our \textit{Wasserstein Lagrangian Flows} are action-minimizing curves for various choices of Lagrangian $\mathcal{L}_i[\qm_t, \dot{\qm}_t,t]$ on the space of densities, which each translate to particular state-space dynamics.
Toy examples of dynamics resulting from various potential or kinetic energy terms are given in (a)-(d).   We may also constrain Wasserstein Lagrangian flows match intermediate data marginals $\qm_{t_i} = \mu_{t_i}$ and combine energy terms to define a suitable notion of interpolation between given $\mu_{t_i}$. }
\vspace{-10pt}
\end{figure*}

However, when inferring trajectories over cell distributions, there exist multiple cell dynamics that yield the same population marginals. This presents an ill-posed problem, which highlights the need for trajectory inference methods that can flexibly incorporate different types of prior information on the cell dynamics.    
Commonly, such prior information is specified via posing a variational problem on the space of marginal distributions, where previous works on measure-valued splines \citep{chen2018measure, benamou2019second, chewi2021fast, clancy2022wasserstein, chen2023deep} are examples which seek to minimize the acceleration of particles.

We propose a general framework for using deep neural networks to infer dynamics and solve marginal interpolation problems, using Lagrangian action functionals on manifolds of probability densities that can flexibly incorporate various types of prior information.
We consider Lagrangians of the form $\cL[\qm_t, \dot{\qm}_t, t] = \cK[\qm_t, \dot{\qm}_t] - \cU[\qm_t, t]$, referring to the first term as a \textit{kinetic energy} and the second as a \textit{potential energy}.
Our
 methods can 
 solve a diverse set of problems defined by the choice of these energies and constraints on the evolution of $\qm_t$.   More explicitly, we specify 
\begin{itemize}
\vspace*{-.12cm}
\item A \textit{kinetic energy} which, in the primary examples considered in this paper, corresponds to a \textit{geometry}
on the space of probability measures.  We primarily consider the Riemannian structures corresponding to the Wasserstein-2 and {Wasserstein Fisher-Rao} metrics.
\vspace*{-.12cm}
\item A \textit{potential energy} 
which is a functional of the density,
such as the expectation of a physical potential 
that encodes 
prior knowledge, or even a nonlinear functional.
\item A collection of \textit{marginal} constraints which are inspired by the availability of data in the problem of interest.   For \gls{OT}, \gls{SB}, or generative modeling tasks, we are often interested in interpolating between two endpoint marginals given by a data distribution and/or a tractable prior distribution.   For applications in {trajectory inference}, we may incorporate multiple constraints to match the
observed temporal marginals, given via data samples.
Notably, in the limit of data sampled infinitely densely in time,  we recover the \gls{AM} framework of \citet{neklyudov2023action}.
\vspace*{-.12cm}
\end{itemize}
We refer to the optimal dynamics specified by these choices as \textit{Wasserstein Lagrangian Flows} and propose a single computational approach to solve a class of such problems, including (i) \textit{standard Wasserstein-2 \gls{OT}} (\cref{example:w2ot}, \citet{benamou2000computational, villani2009optimal}), (ii) \textit{entropy regularized \gls{OT}} or Schrödinger Bridge (\cref{example:sb}, \citet{leonard2013survey, chen2021stochastic}, (iii) \textit{physically constrained \gls{OT}} (\cref{example:pot}, \citet{tong2020trajectorynet,koshizuka2022neural}), and (iv) \textit{unbalanced \gls{OT}} (\cref{example:wfr}, \citet{chizat2018interpolating}) (\cref{sec:examples}).
Our framework also allows for combining energy terms to incorporate features of the above problems as inductive biases for trajectory inference.  In \cref{sec:experiments}, we showcase the ability of our methods to accurately solve Wasserstein Lagrangian flow optimizations, and highlight how exploring different Lagrangians can improve results in single-cell RNA-sequencing applications.
We compare our approach
to related work in \cref{sec:related}.
\vheader
\vspace*{-.12cm}
\section{Background}\label{sec:bg}
\vspace*{-.12cm}
\vheader

\subsection{Wasserstein-2 Geometry}\label{sec:ot}
\vspace*{-.12cm}

For two given densities with finite second moments $\mu_0, \mu_1 \in \wmanifold$,  the Wasserstein-2 \gls{OT} problem is defined
as a cost-minimization problem over joint distributions or `couplings' $\pi \in \Pi(\mu_0,\mu_1) = \{\pi(x_0,x_1) \cond  \pi(\x) = \mu_0(\x), \; \pi(\y) = \mu_1(\y)\}$, i.e.
\begin{align}
\begin{split}\label{eq:ot}
\scalebox{.95}{\ensuremath{ 
    W_2(\mu_0, \mu_1)^2  \coloneqq 
    \inf \limits_{\pi \in \Pi(\mu_0,\mu_1)} ~\int 
    \| \x - \y \|^2
    \pi(\x,\y) \dd\x \dd\y. 
    }}
\end{split}
\end{align}
\normalsize
The dynamical formulation of \citet{benamou2000computational} gives an alternative perspective on the $W_2$ \gls{OT} problem as an optimization over a vector field $\vel_t$ that transports samples according to an \textsc{ODE} $\dot{x}_t = \vel_t$. 
The evolution of the samples' density $\qm_t$,
under transport by $v_t$, is governed by the \textit{continuity equation} $\dot{\qm}_t = - \divp{ \qm_t v_t }$
(\citet{figalli2021invitation} Lemma 4.1.1), and we have
\begin{align}
\begin{split} \label{eq:dynamical} 
   \hspace*{-.2cm} \frac{1}{2} W_2(\mu_0, \mu_1)^2
    = \inf \limits_{\qm_t}\inf  \limits_{v_t} \int_0^1  \int \frac{1}{2} 
    {\xmetricnorm{v_t}^2} ~\qm_t \dx \dt \\
    \st  \dot{\qm}_t = - \divp{ \qm_t v_t }, \,\,
    ~ \qm_0 = \mu_0, \,\,  
           \qm_1 = \mu_1 ,
    \end{split}
\end{align}
\normalsize
where $\divp{  }$ is the divergence operator.
The $W_2$ transport cost can be viewed as providing a Riemannian manifold structure on 
$\wmanifold$ (\citet{otto2001geometry, ambrosio2008gradient}, see also \citet{figalli2021invitation} Ch. 4).   
Introducing Lagrange multipliers $\lagr_t$ to enforce the constraints in \cref{eq:dynamical}, we obtain the condition $\vel_t = \nabla \lagr_t$ (see \cref{app:dual_kinetic}), which is suggestive of the result from \citet{ambrosio2008gradient} characterizing the tangent space 
{$T^{W_2}_{\qm}\cP_2 = \{ 
\dot{\qm}~ | \int \dot{\qm}\dx = 0 \}$ via the continuity equation},
\begin{align}
T^{W_2}_{\qm}\wmanifold = \{
\dot{\qm}
~|~  \dot{\qm}  = -\divp{ \qm \nabla \lagr } \}. 
\label{eq:w2_tangent}
\end{align}
\normalsize
{We also write the cotangent space as $T^{*W_2}_{\qm}\wmanifold = \{ [\lagr ] ~ | ~ \lagr \in \cC^\infty(\mathcal{X})\}$,}
where $\cC^\infty(\mathcal{X})$ denotes smooth functions and $[\lagr]$ is an equivalence class up to addition by a constant.
For two curves $\mu_t, \qm_t : [-\epsilon,\epsilon] \mapsto \wmanifold$ passing through $\qm \coloneqq  \qm_0 = \mu_0$,  the Otto metric is defined
\begin{align}\label{eq:w2_metric}
\wmetricat{ \dot{\mu}_t, \dot{\qm}_t }{\qm} = \cometricat{ \lagr_{\dot{\mu}_t}, \lagr_{\dot{\qm}_t} }{\qm}^{W_2} = \int \langle \nabla \lagr_{\dot{\mu}_t}, \nabla \lagr_{\dot{\qm}_t}\rangle \qm~\dx .
\end{align}
 \normalsize
\subsection{Wasserstein Fisher-Rao Geometry}\label{sec:wfr}
Building from the dynamical formulation in \cref{eq:dynamical}, \citet{chizat2018interpolating, chizat2018unbalanced, kondratyev2016new, liero2016optimal, liero2018optimal} consider additional terms allowing for birth and death of particles or teleportation of probability mass.  
In particular, we extend the continuity equation with
a `growth' term $g_t: \cX \rightarrow \mathbb{R}$ whose square is regularized in the cost,
\small
\begin{align}
 & \frac{1}{2}WFR_\lambda(\mu_0, \mu_1)^2
    = \inf  \limits_{\qm_t} \inf \limits_{v_t,g_t} \int_0^1  \int \left( \frac{1}{2} \| v_t \|^2 + \frac{\lambda}{2} g_t^2 \right) \qm_t  \dx \dt, 
    \nonumber \\
    &\st \quad \dot{\qm}_t = - \divp{ \qm_t v_t } + \lambda \qm_t g_t, \,\, \qm_0 = \mu_0, \,\, \qm_1 = \mu_1 .\label{eq:wfr}
\end{align}
\normalsize
We call this the \gls{WFR} distance, since considering \textit{only} the growth terms recovers
the non-parametric Fisher-Rao metric \citep{chizat2018interpolating,bauer2016uniqueness}.   We also refer to \cref{eq:wfr} as an \textit{unbalanced} \gls{OT} problem on the space of unnormalized densities $\cM(\cX)$, since the growth terms need not preserve normalization $\int \dot{\qm}_t \dx = \int \lambda g_t \qm_t \dx \neq 0$ without further modifications (see e.g. \citet{lu2019accelerating}).   

\citet{kondratyev2016new} 
define a Riemannian structure on $\cM(\cX)$ via the \gls{WFR} distance.  Introducing Lagrange multipliers $\lagr_t$ and eliminating $v_t, g_t$ in \cref{eq:wfr} yields the optimality conditions 
$v_t = \nabla \lagr_t $ and $g_t = \lagr_t$.
In analogy with \cref{sec:ot}, this suggests characterizing the tangent space via the tuple 
$(\lagr_t, \nabla \lagr_t)$ and defining the metric as
a characterization of the tangent space
\small
\begin{align}
T^{WFR_\lambda}_{\qm}\wfrmanifold &= \{ 
\dot{\qm}
~|~  \dot{\qm}  = -\divp{  \qm \nabla \lagr } + \lambda \qm \lagr \} 
\qquad 
\label{eq:wfr_tangent} \\[1.25ex]
\begin{split}
\wfrmetricat{ \dot{\mu}_t, \dot{\qm}_t }{\qm} &=  \wfrcometricat{  \lagr_{\dot{\mu}_t},\lagr_{\dot{\qm}_t} }{\qm} \\ 
&= \int \big(\vphantom{\frac{1}{2}} \langle \nabla \lagr_{\dot{\mu}_t}, \nabla \lagr_{\dot{\qm}_t}\rangle  + \lambda ~ \lagr_{\dot{\mu}_t} \lagr_{\dot{\qm}_t}\big) \qm \dx. 
 \end{split}\label{eq:wfr_metric}
\end{align}
\normalsize

\newcommand{\ampath}{\mu}
\vheader
\vheader 
\subsection{Action Matching}
Finally, Action Matching (\gls{AM}) \citep{neklyudov2023action} considers only the inner optimizations in \cref{eq:dynamical} or \cref{eq:wfr} as a function of $v_t$ or $(v_t, g_t)$, assuming a distributional path $\ampath_t$ is given via samples. 
In the $W_2$ case, to solve for the velocity $v_t = \nabla \lagr_{\dot{\ampath}_t}$ which corresponds to $\ampath_t$ via the continuity equation or \cref{eq:w2_tangent}, \citet{neklyudov2023action} optimize the objective
\begin{align}
\begin{split}
\cA[\ampath_t] =&~  \sup \limits_{\lagr_t} \int \lagr_1 \ampath_1 \dx - \int \lagr_0  \ampath_0 \dx \\ 
&~- \int_0^1 \int \left( \frac{\partial \lagr_t}{\partial t} + \frac{1}{2}\| \nabla \lagr_t\|^2 \right) \ampath_t \dx \dt, 
\end{split} \label{eq:action_matching}
\end{align}
\normalsize
over $\lagr_t: \cX \times [0,1] \rightarrow \bbR$ parameterized by a neural network, with similar objectives for $WFR_{\lambda}$.
To foreshadow our exposition in \cref{sec:wlm}, we view Action Matching as maximizing a lower bound on the 
\textit{action} $\cA[\ampath_t]$ or \textit{kinetic energy} of the curve $\ampath_t: [0,1] \rightarrow \wmanifold$ of densities.
In particular, at the optimal $\lagr_{\dot{\ampath}_t}$ satisfying $\dot{\ampath}_t =  -\divp{ \ampath_t \nabla \lagr_{\dot{\ampath}_t} }$, the value of \cref{eq:action_matching} becomes
\begin{align}
\begin{split}
\cA[\ampath_t] &= \int_0^1 \frac{1}{2} \wmetricat{ \dot{\ampath}_t, \dot{\ampath}_t }{\ampath_t} \dt = \int_0^1 \frac{1}{2} \wcometricat{ \lagr_{\dot{\ampath}_t}, \lagr_{\dot{\ampath}_t}}{\ampath_t} \dt \\
&= \int_0^1 \int \frac{1}{2} \|\nabla \lagr_t \|^2 \mu_t \dx \dt.
\end{split} \label{eq:action_kinetic} 
\end{align}
\normalsize
Our proposed framework
considers minimizing the action functional over distributional paths, and our computational approach will include \gls{AM} as a component.   

\vheader
\section{Wasserstein Lagrangian Flows}\label{sec:wlm}
\vheader
In this section, we develop computational methods for optimizing Lagrangian action functionals on the space of (unnormalized) densities $\cP(\cX)$.\footnote{For convenience, we describe our methods using a generic $\cP(\cX)$ (which may {represent $\wmanifold$ or $\wfrmanifold$).}}
Lagrangian actions are commonly used to define a cost function on the ground space $\cX$, which is then `lifted' to the space of densities via an optimal transport distance (\citet{villani2009optimal} Ch. 7).   

We instead propose to formulate Lagrangians $\cL[\qm_t, \dot{\qm}_t, t]$ \textit{directly} in the density space, which 
includes \gls{OT} with ground-space Lagrangian costs as a special case (\cref{app:lagr}), but \textit{also} allows us to consider kinetic and potential energy functionals which depend on the density and thus cannot be expressed using a ground-space Lagrangian. 
{In particular, we consider kinetic energies capturing space-dependent birth-death terms (as in $WFR_{\lambda}$, \cref{example:wfr}) 
and potential energies capturing global information about the distribution of particles or cells (as in the \gls{SB} problem,
\cref{example:sb}).
}

\begin{figure*}
\centering
    \includegraphics[width=.95\textwidth]{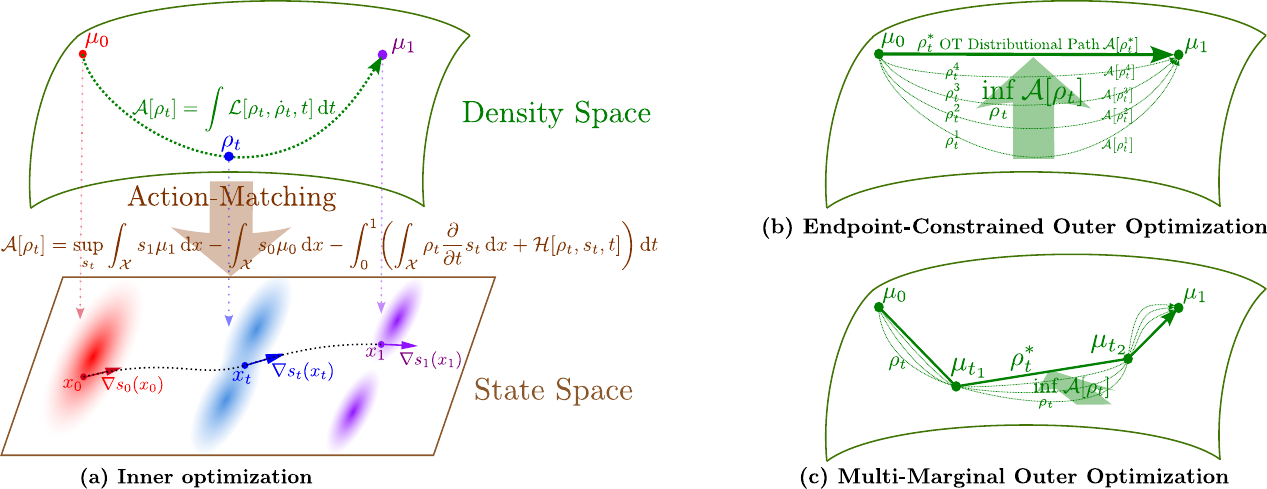}
  \vspace*{-.15cm}  
    \caption{{
    For different definitions of Lagrangian $\mathcal{L}[\qm_t, \dot{\qm}_t,t]$ or Hamiltonian $\mathcal{H}[\qm_t, \lagr_t,t]$
    on the space of densities, 
    we obtain different action functionals $\mathcal{A}[\qm_t]$.   Here, we show state-space velocities and optimal density paths for the $W_2$ geometry and \gls{OT} problem.
    (a) The action functional for each curve can be evaluated using Action Matching (inner optimization in \cref{prop:dual_obj}), which is performed in the state-space.  (b,c)  Minimization of the action functional (outer optimization in \cref{prop:dual_obj}) is performed on the space of densities satisfying two endpoint constraints and possible intermediate constraints.}
    }
    \label{fig:ot-lagrangian}
  \vspace*{-.2cm} 
\end{figure*}

 \vheader
 \vheader
\subsection{Wasserstein Lagrangian and Hamiltonian Flows}\label{sec:abstract_approach} 
 \vheader
We consider Lagrangian action functionals on the space of densities, defined in terms of a \textit{kinetic energy} $\cK[\qm_t, \dot{\qm}_t]$, which captures any dependence on the velocity of a curve $\dot{\qm}_t$, and a \textit{potential energy} $\cU[\qm_t, t]$ which depends only on the position $\qm_t$,
\begin{align}\label{eq:lagr0}
 \mathcal{L}[\qm_t, \dot{\qm}_t,t]  &= \ifhl{\cK[\qm_t, \dot{\qm}_t]} - \cU[\qm_t, t]  .
\end{align}
\normalsize
Throughout, we will assume $\mathcal{L}[\qm_t, \dot{\qm}_t,t]$ is lower semi-continuous (lsc) and strictly convex in $\dot{\qm}_t$.


Our goal is to solve for \textit{Wasserstein Lagrangian Flows}, by optimizing the given Lagrangian over curves of densities $\qm_t : [0,1] \rightarrow \cP(\cX)$ which are constrained to pass through $M$ given points $\mu_{t_i} \in \cP(\cX)$ at times $t_i$.
We define the \textit{action} of a curve $\cA_{\cL}[\qm_t] = \int_0^1 \mathcal{L}[\qm_t, \dot{\qm}_t,t]\dt$
as the time-integral of the Lagrangian and seek the action-minimizing curve subject to the constraints 
\begin{align}
&\cS_{\cL}(\{\mu_{t_i}\}_{i=0}^{M-1}) \coloneqq  \inf
\limits_{\qm_t \in \mmarginalcouple} \cA_{\cL}[\qm_t] \label{eq:general_lagr_min} \\
&= \inf \limits_{\qm_t}
\int_0^1 \mathcal{L}[\qm_t, \dot{\qm}_t,t]\dt 
\st \,
\qm_{t_i}   = \mu_{t_i} \,\, \scalebox{.9}{ \text{$\forall \,\, 0\leq i \leq M-1$} } \nonumber
\end{align}
\normalsize
where $\mmarginalcouple = \{ \qm_t : [0,1] \rightarrow \cP(\cX) ~|~  \qm_0 = \mu_0, ~ \qm_1  = \mu_1, ~ \qm_{t_i}   = \mu_{t_i} \quad (\forall \,\, 1\leq i \leq M-2) \}$ indicates the set of curves matching the given constraints.  
We proceed writing the value as $\cS$,
leaving the dependence on $\cL$ and $\mu_{t_i}$ implicit, but note $\cS(\mu_{0,1})$ for $M=2$ as an important special case. 
Our objectives for solving \cref{eq:general_lagr_min} 
are based on the Hamiltonian $\cH$ associated with the chosen Lagrangian. 
In particular,  consider a cotangent vector $\lagr_t \in T_{\qm_t}^*\wmanifold$ {or $\lagr_t \in T_{\qm_t}^*\wfrmanifold$}, which is identified with a linear functional on the tangent space $\lagr_t[\cdot] : \dot{\qm}_t \mapsto \int \lagr_t \dot{\qm_t} \dx$ via the canonical duality bracket.  We define the \textit{Hamiltonian} $\cH[\qm_t, \lagr_t, t]$ via the Legendre transform
\begin{align}
\begin{split} \label{eq:hamiltonian}
 \cH[\qm_t, \lagr_t, t] =~& \sup \limits_{\dot{\qm}_t \in T_{\qm_t}\cP} \int \lagr_t \dot{\qm_t} \dx - \mathcal{L}[\qm_t, \dot{\qm}_t,t]  \\
 =~&  \ifhl{\cK^*[\qm_t, \lagr_t] } + \cU[\qm_t, t] ,
 \end{split}
\end{align}
\normalsize
where the sign of $\cU[\qm_t, t]$ changes and $\cK^*[\qm_t, \lagr_t]$ translates the kinetic energy to the dual space.   
A primary example is when $\cK[\qm_t, \dot{\qm}_t] = \frac{1}{2}\tmetricat{\dot{\qm}_t, \dot{\qm}_t}{\qm_t}$ is given by a Riemannian metric in the tangent space (such as for $W_2$ or {$WFR_\lambda$}), then $\cK^*[\qm_t, \lagr_t] = \frac{1}{2}\cometricat{\lagr_t, \lagr_t}{\qm_t}$ is the same metric written in the cotangent space  
(see \cref{app:dual_kinetic} for detailed derivations for all examples considered in this work).

Finally, under our assumptions, $\mathcal{L}[\qm_t, \dot{\qm}_t,t]$ can be written using the Legendre transform, 
\begin{align}
\mathcal{L}[\qm_t, \dot{\qm}_t,t]  = \sup \limits_{\lagr_t \in T^*_{\qm_t}\cP} \int \lagr_t \dot{\qm_t} \dx - \cH[\qm_t, \lagr_t, t] . \label{eq:l_legendre}
\end{align}
\normalsize
The following theorem forms the basis for our computational approach, and can be derived using
this Legendre transform
and integration by parts in time (see \cref{app:dual} for proof and \cref{fig:ot-lagrangian} for visualization).

\begin{restatable}{theorem}{wlfdual}\label{prop:dual_obj}
For a Lagrangian $\mathcal{L}[\qm_t, \dot{\qm}_t,t]$ which is lsc and strictly convex in $\dot{\qm}_t$, the optimization
\begin{align}
\cS
&= \inf \limits_{\qm_t \in \mmarginalcouple} \cA_{\cL}[\qm_t] = \inf \limits_{\qm_t \in \mmarginalcouple} \int_0^1 \mathcal{L}[\qm_t, \dot{\qm}_t,t]  \dt  \nonumber
\intertext{is equivalent to the following dual}
 \cS =&~\inf \limits_{\qm_t \in \mmarginalcouple}  \sup \limits_{\lagr_t}  \int \lagr_1 \mu_1 \dx -  \int \lagr_0 \mu_0 \dx \label{eq:dual_intermediate} \\
 &
 \hphantom{\inf \limits_{\qm_t \in \mmarginalcouple}} -  \int_0^1  \left( \int  \frac{\partial \lagr_t }{\partial t}  \qm_t \dx  + \cH[\qm_t, \lagr_t, t] \right)  \dt \nonumber
\end{align}
\normalsize
where, for $\lagr_t \in T^*_{\qm_t}\cP$, the Hamiltonian $\cH[\qm_t, \lagr_t, t] $ is the Legendre transform of $\mathcal{L}[\qm_t, \dot{\qm}_t,t]$ (Eq \ref{eq:hamiltonian}).

The action $\cA_{\cL}[\qm_t]$ of a given 
curve is the solution to the inner optimization, 
\begin{align}
\begin{split}
\cA_{\cL}[\qm_t] =&~ \sup \limits_{\lagr_t}   \int \lagr_1 \mu_1 \dx -  \int \lagr_0 \mu_0 \dx \\
&~~-  \int_0^1  \left( \int  \frac{\partial \lagr_t }{\partial t}  \qm_t \dx  + \cH[\qm_t, \lagr_t, t] \right)  \dt .
\end{split}\label{eq:action}
\end{align}
\normalsize
\end{restatable}
In line with our goal of defining Lagrangian actions directly on $\cP(\cX)$ instead of via $\cX$,  \cref{prop:dual_obj} operates only in the abstract space of densities.  See \cref{app:equivalence} for 
discussion.

Finally, 
the solution to 
\cref{eq:general_lagr_min} can also be expressed in a Hamiltonian perspective, where the resulting \textit{Wasserstein Hamiltonian flows} \citep{chow2020wasserstein} satisfy optimality conditions $\frac{\partial \qm_t}{\partial t} = \frac{\delta }{\delta \lagr_t} \cH[\qm_t, \lagr_t, t]$ and $\frac{\partial \lagr_t}{\partial t} = -\frac{\delta }{\delta \qm_t} \cH[\qm_t, \lagr_t, t]$.

To further analyze \cref{prop:dual_obj} and set the stage for our computational approach in \cref{sec:computational}, we 
consider the two optimizations in \cref{eq:dual_intermediate} as (i)
\textbf{\textit{evaluating}} 
the action functional $\cA_{\cL}[\qm_t]$ given a curve $\qm_t$, and (ii) \textbf{\textit{optimizing}} the action over curves $\qm_t \in \mmarginalcouple$ satisfying the desired constraints.   

\vheader 
\subsubsection{Inner Optimization:  Evaluating $\cA_{\cL}[\qm_t]$ using Action Matching}\label{sec:eval}
\vheader
We immediately recognize the similarity of \cref{eq:action} to the \gls{AM} objective in \cref{eq:action_matching} for $\cH[\qm_t,\lagr_t,t] = \int \frac{1}{2}\| \nabla \lagr_t\|^2 \qm_t \dx$, 
which suggests a generalized notion of
Action Matching as an inner loop to evaluate $\cA_{\cL}[\qm_t]$ for a given $\qm_t \in \mmarginalcouple$ in 
\cref{prop:dual_obj}.    
For all $t$, the optimal cotangent vector $\lagr_{\dot{\qm}_t}$ corresponds to the tangent vector $\dot{\qm}_t$ of the given curve via the Legendre transform or \cref{eq:action}. 

\citet{neklyudov2023action} assume access to samples from a \textit{continuous} curve of densities $\mu_{t}$ which, from our perspective, corresponds to the limit as the number of constraints $M \rightarrow \infty$.  Since $\qm_t \in \mmarginalcouple$ has no remaining degrees of freedom in this case, the outer optimization over $\qm_t$ can be ignored  and expectations in \cref{eq:action_matching} are written directly under $\mu_t$.
However, this assumption is often unreasonable in applications such as trajectory inference, where data is sampled discretely in time.

\vheader
\subsubsection{Outer Optimization over Constrained Distributional Paths}\label{sec:outer_curves}
\vheader

In our settings of interest, the outer optimization over curves 
$\cS_\cL(\{\mu_{t_i}\}) = \inf_{\qm_t \in \mmarginalcouple} \cA_{\cL}[\qm_t]$
is thus necessary to \textit{interpolate} between $M$ given marginals using the inductive bias encoded in the Lagrangian $\cL[\qm_t, \dot{\qm}_t, t]$.  
Crucially, our parameterization of $\qm_t$ in \cref{sec:parametrization} enforces $\qm_t \in \mmarginalcouple$ by design, given access to samples from $\mu_{t_i}$.
{Nevertheless, upon reaching an optimal $\qm_t$, our primary object of interest is the dynamics model corresponding to $\dot{\qm}_t$ and parameterized by the optimal $\lagr_{\dot{\qm}_t}$ in \cref{eq:action},
which may be used to transport particles or predict individual trajectories.}

\subsection{Computational Approach 
for Solving Wasserstein Lagrangian Flows
}\label{sec:computational}
In this section, we describe our computational approach to solving for a class of Wasserstein Lagrangian Flows, which is summarized in \cref{algorithm:wlm}.

\subsubsection{Linearizable Objectives}
Despite the generality of \cref{prop:dual_obj}, we restrict attention to Lagrangians with the following property.
\begin{definition}[(Dual) Linearizability]\label{def:dual_linearizable}
 A Lagrangian $\cL[\qm_t, \dot{\qm}_t, t]$ is \textup{dual linearizable} if the corresponding Hamiltonian 
 $\cH[\qm_t, \lagr_t, t]$ can be written as a linear functional of the density $\qm_t$. \ifhl{In other words, $\cH[\qm_t, \lagr_t, t]$ is \textup{linearizable} if there 
 exists a function $\mathbb{H}(x,s_t, t)$}
 s.t. 
 \vspace*{-.1cm}
  \begin{align}
\cH[\qm_t, \lagr_t, t] 
\coloneqq~& \int \ifhl{\mathbb{H}(x,\lagr_t,t)\qm_t(x)} \dx\,. 
 \end{align}
\normalsize
 \end{definition}
  \vspace*{-.1cm}
This property suggests that we only need to draw samples from $\qm_t$ and need not evaluate its density, which allows us to derive an efficient parameterization of curves 
satisfying $\qm_t \in \mmarginalcouple$ below. 
\footnote{In \cref{app:other_examples}, we highlight the Schrödinger Equation
as a special case 
of our framework which 
does not appear to admit a linear dual problem.   
In this case, 
optimization of \cref{eq:dual_intermediate} may require explicit modeling of the density $\qm_t$ corresponding to a given set of particles $x_t$ (e.g. see \citet{pfau2020ab}).}

{As examples, note that the $WFR_{\lambda}$ or $W_2$ metrics as the Lagrangian yield a linear Hamiltonian $\cH[\qm_t, \lagr_t, t] = \cK^*[\qm_t, \lagr_t] 
=\int ( \frac{1}{2} \| \nabla \lagr_t\|^2 + \frac{\lambda}{2} \lagr_t^2 )\qm_t \dx$, with $\lambda=0$ for $W_2$.}
Potential energies $\cU[\qm_t,t] = \int V_t (x) \qm_t \dx$ which are linear in $\qm_t$ 
(\cref{example:pot}) clearly satisfy \cref{def:dual_linearizable}.   However, nonlinear potential energies (as in \cref{example:sb}) will require reparameterization to be linearizable.  

\newcommand{\myline}{\\[.9ex]}
\newcommand{\of}[1]{\big(#1\big)}
\setlength{\textfloatsep}{20pt}
\begin{algorithm}[t]
\small
  \begin{algorithmic}
    \caption{Learning Wasserstein Lagrangian Flows}\label{algorithm:wlm}
    \REQUIRE samples from the marginals $\mu_0, \mu_1$, parametric model $\lagr_t(x,\theta)$, generator from $\qm_t(x,\eta)$\myline
    \FOR{learning iterations}
        \STATE sample from marginals $\{x_0^i\}_{i=1}^n \sim \mu_0, \; \{x_1^i\}_{i=1}^n \sim \mu_1$ \myline
        \STATE sample time $\{t^i\}_{i=1}^n \sim \textsc{uniform}[0,1]$ \myline
        \STATE $x_t^i = (1-t^i)x_0^i + t^ix_1^i + t^i(1-t^i)\textsc{nnet}(t^i,x_0^i,x_1^i;\eta)$ \myline
        \STATE $\textsc{grad}_\eta = -\nabla_\eta \frac{1}{n}\sum_i^n\big[\frac{\partial \lagr_{t^i}}{\partial t}(x_t^i(\eta),\theta) +$ \myline
        \STATE $\qquad\qquad+ \mathbb{H}\of{x_t^i(\eta),\lagr_{t^i}(x_t^i(\eta),\theta),t^i}\big]$ \myline
        \FOR{Wasserstein gradient steps}
        \STATE $x_t^i \gets x_t^i + \alpha \cdot t^i (1-t^i)\nabla_x \big[\frac{\partial \lagr_{t^i}}{\partial t}(x_t^i,\theta)$ \myline
        \STATE $\qquad \quad +
        \mathbb{H}\of{x_t^i,\lagr_{t^i}(x_t^i,\theta),t^i}\big]$ \myline
        \ENDFOR
        \STATE $\textsc{grad}_\theta = \nabla_\theta \frac{1}{n}\sum_{i}^n\big[\lagr_1(x_1^i,\theta) - \lagr_0(x_0^i,\theta) - \frac{\partial \lagr_{t^i}(x_t^i,\theta)}{\partial t}$ \myline
        \STATE $\qquad\qquad-\mathbb{H}\of{x_t^i,\lagr_{t^i}(x_t^i,\theta),t^i}\big]$ \myline
        \STATE update parameters $\eta$ using gradient descent with $\textsc{grad}_\eta$ \myline
        \STATE update parameters $\theta$ using  gradient ascent with $\textsc{grad}_\theta$
    \ENDFOR
    \OUTPUT cotangent vectors $\lagr_t(x,\theta)$
  \end{algorithmic}
\end{algorithm}

\subsubsection{Parameterization and Optimization}\label{sec:parametrization}
For any Lagrangian optimization with a linearizable dual objective as in \cref{def:dual_linearizable}, we consider
parameterizing the cotangent vectors $\lagr_t$ and the distributional path $\qm_t$.  
We parameterize $\lagr_t$ as a neural network $\lagr_t(x,\theta)$ which takes $t$ and $x$ as inputs with parameters $\theta$, and outputs a scalar.
Inspired by the fact that we only need to draw samples from $\qm_t$ for these problems,
we parameterize the distribution path $\qm_t(x,\eta)$ as a generative model, where the samples are generated as follows
\begin{align}
\begin{split}
    x_t &= (1-t)x_0 + tx_1 + t(1-t)\textsc{nnet}(t,x_0,x_1;
    \eta), \\ & x_0 \sim \mu_0, \qquad \;\; x_1 \sim \mu_1.
    \end{split} \label{eq:nn_two} \raisetag{8pt}
\end{align}
\normalsize
Notably, this preserves the endpoint marginals $\mu_0, \mu_1$. 
For multiple constraints, we can modify our sampling procedure to interpolate between two intermediate dataset marginals, with neural network parameters $\eta$ shared across timesteps
\small
\begin{align}
   x_t=~& \tilde{\beta}_t x_{t_i} + \hat{\beta}_t x_{t_{i+1}} + \left(1- \tilde{\beta}_t^2 - \hat{\beta}_t^2\right)\textsc{nnet}(t,x_{t_i},x_{t_{i+1}};\eta),\nonumber\\
    &\tilde{\beta}_t = \frac{t_{i+1}-t}{t_{i+1}-t_i}, \;\;\;\hat{\beta}_t=\frac{t-t_{i}}{t_{i+1}-t_i}\,.
\end{align}
\normalsize
For linearizable dual objectives as in \cref{{eq:dual_intermediate}} and \cref{def:dual_linearizable}, we optimize
\ifhl{
\small
\begin{align}
\label{eq:dual_for_optimization}
&\min \limits_{\eta}\max \limits_{\theta} 
\int \lagr_1(x,\theta) \mu_1 \dx -  \int \lagr_0(x,\theta) \mu_0 \dx \\
&- \int_0^1 \int 
\left( \frac{\partial \lagr_t}{\partial t}(x,\theta) + 
\mathbb{H} \of{x,\lagr_t(x,\theta),t} \right)
\qm_t(x,\eta)\dx \dt,    \nonumber
\end{align}
} 
\normalsize
where the optimization w.r.t. $\eta$ is performed via the re-parameterization trick.

\paragraph{Wasserstein Gradient Approach} An alternative to parametrizing the distributional path $\qm_t$ is 
to perform minimization of \cref{eq:dual_for_optimization}
via the Wasserstein gradient flow, i.e. the samples $x_t$ from the initial path $\qm_t$ are updated by 
\small
\begin{align}
    x_t' = x_t + \alpha \cdot  t &(1-t)\nabla_x \bigg[\frac{\partial \lagr_t}{\partial t}(x_t,\theta) + \mathbb{H}\of{x_t,\lagr_t(x_t,\theta),t}\bigg]\,, \nonumber
\end{align}
\normalsize
where $\alpha$ is a hyperparameter regulating the step-size, and the coefficient $t(1-t)$ guarantees the preservation of the endpoints. 
In practice, we find that combining both 
the parametric and nonparametric
approaches works best.
The pseudo-code for the resulting algorithm is given in \cref{algorithm:wlm}. 

\paragraph{Discontinuous Interpolation}
The support of the optimal distribution path $\qm_t$ might be disconnected, as in \cref{fig:examples_a}.
Thus, it may be impossible to interpolate continuously between independent
samples from the marginals while staying in the support of the optimal path.
To allow for a discontinuous interpolation, we pass a discontinuous indicator variable $\mathbbm{1}[t < 0.5]$ to the model $\qm_t(x,\eta)$.
This indicator is crucial to ensure our parameterization is expressive enough to  approximate any suitable distributional path including, for example, the optimal \gls{OT} path (see proof in \cref{app:expressivity}).



\begin{restatable}{proposition}{express}
\label{prop:expressivity}
   For any absolutely-continuous distributional path $\rho_t :[0,1] \mapsto \wmanifold$ on the $W_2$ manifold, there exists a function $\textsc{nnet}^*(t, x_0, x_1, \mathbbm{1}[t < 0.5]; \eta)$ such that \cref{eq:nn_two} samples from $\rho_t$.
\end{restatable}

\vheader

\vheader
\vheader
\section{Examples of Wasserstein Lagrangian Flows}\label{sec:examples}
\vheader
We now analyze the Lagrangians, dual objectives, and Hamiltonian optimality conditions corresponding to several important examples of Wasserstein Lagrangian flows.
We present various kinetic and potential energy terms using their motivating examples and $M=2$ endpoint constraints.
However, 
we may combine 
various 
energy terms to construct Lagrangians $\cL[\qm_t,\dot{\qm}_t,t]$
and optimize subject to multiple constraints, as we do in 
our experiments in \cref{sec:experiments}.



\begin{example}[\textbf{$W_2$ Optimal Transport}]\label{example:w2ot}
The Benamou-Brenier formulation of $W_2$ optimal transport in \cref{eq:dynamical} is the simplest example of our framework, with no potential energy and the kinetic energy defined by the Otto metric $\cL[\qm_t, \dot{\qm}_t, t] = \frac{1}{2} \wmetricat{ \dot{\qm}_t, \dot{\qm}_t }{\qm_t} = \cH[\qm_t, \lagr_{\dot{\qm}_t}, t] = \frac{1}{2} \int \| \nabla \lagr_{\dot{\qm}_t} \|^2 \qm_t \dx$.   
In contrast to \cref{eq:dynamical}, note that our Lagrangian optimization in \cref{eq:general_lagr_min} is over $\qm_t$ only, while solving the dual objective introduces the second optimization to identify $\lagr_{\dot{\qm}_t}$ such that $\dot{\qm}_t = -\divp{ \qm_t \nabla \lagr_{\dot{\qm}_t}}$.
Our objective for solving \gls{OT} 
problem with quadratic cost is 
\begin{align}
\begin{split}
    \cS_{OT} =&~  \inf \limits_{\qm_t \in \marginalcouple} \sup \limits_{\lagr_t} ~ \int \lagr_1 \mu_1 \dx-  \int \lagr_0 \mu_0 \dx
  \\
  &- \int_0^1  \int \left( \frac{\partial \lagr_t }{\partial t}  + \frac{1}{2} \| \nabla \lagr_t \|^2 \right) \qm_t   \dx \dt ,
  \end{split}
\end{align}
\normalsize
where the Hamiltonian optimality conditions 
$\frac{\partial \qm_t}{\partial t} = \frac{\delta }{\delta \lagr_t} \cH[\qm_t, \lagr_t, t]$ and $\frac{\partial \lagr_t}{\partial t} = -\frac{\delta }{\delta \qm_t} \cH[\qm_t, \lagr_t, t]$ \citep{chow2020wasserstein} recover the well-known characterization of the Wasserstein geodesics via the continuity and Hamilton-Jacobi equations \citep{benamou2000computational},
\begin{align}
\dot{\qm_t} = - \divp{ \qm_t \nabla \lagr_t } \qquad  \frac{\partial \lagr_t }{\partial t}  + \frac{1}{2} \| \nabla \lagr_t \|^2 = 0 . \label{eq:ot_opt}
\end{align}
\normalsize
It is well known that optimal transport plans (or Wasserstein-2 geodesics) are `straight-paths' in the Euclidean space \citep{villani2009optimal}. For the flow induced by a vector field $\nabla \lagr_t$, we calculate the acceleration, or second time derivative, as
\begin{align}
   \ddot{X}_t &= \nabla \left[ \frac{\partial  \lagr_t}{\partial t} + \frac{1}{2} \| \nabla \lagr_t \|^2 \right] = 0 ,  \label{eq:ot_accel1}
\end{align}
\normalsize
where zero acceleration is achieved if $\frac{\partial \lagr_t }{\partial t}  + \frac{1}{2} \| \nabla \lagr_t \|^2 = c, \forall t$, as occurs at optimality in \cref{eq:ot_opt}.
\end{example}

\begin{table*}[t]
    \caption{Results for high-dim PCA representation of single-cell data for corresponding datasets. We report Wasserstein-$1$ distance averaged over left-out marginals, averaged over $5$ independent runs. Results with citations are taken from corresponding papers.}
    \vspace{-10pt}
    \label{tab:manydim}
    \begin{center}
    \resizebox{\textwidth}{!}{
    \begin{tabular}{llllll}
         & \multicolumn{1}{c}{\bf dim=5} & \multicolumn{2}{c}{\bf dim=50} & \multicolumn{2}{c}{\bf dim=100} \\
         \multicolumn{1}{c}{\bf Method} & 
         \multicolumn{1}{c}{\bf EB} &
         \multicolumn{1}{c}{\bf Cite} & \multicolumn{1}{c}{\bf Multi} & \multicolumn{1}{c}{\bf Cite} & \multicolumn{1}{c}{\bf Multi} \\
          \hline \\
         exact OT & $0.822 $& $37.569$ & $47.084$ & $42.974$ & $53.271$ \\
         WLF-OT (ours) & $0.814 \pm 0.002 $ & $38.253 \pm 0.071$ & $47.736 \pm 0.110$ & $44.769 \pm 0.054$ & $55.313 \pm 0.754$ \\
         OT-CFM (more parameters) & $0.822 \pm 3.0\text{e-}4 $ & $37.821 \pm 0.010$ & $47.268 \pm 0.017$ & $44.013 \pm 0.010$ & $54.253 \pm 0.012$ \\
         OT-CFM \citep{tong2023improving} & $\mathbf{0.790 \pm 0.068}$  & $38.756 \pm 0.398$ & $47.576 \pm 6.622$ & $45.393 \pm 0.416$ & $54.814 \pm 5.858$ \\
         I-CFM \citep{tong2023improving} & $0.872 \pm 0.087$ & $41.834 \pm 3.284$ & $49.779 \pm 4.430$ & $48.276 \pm 3.281$ & $57.262 \pm 3.855$ \\
         \hline \\
         WLF-UOT ($\lambda = 1$, ours) &  $\mathbf{0.800 \pm 0.002 }$ & $\mathbf{37.035 \pm 0.079}$ & $45.903 \pm 0.161$ & $\mathbf{43.530 \pm 0.067}$ & $53.403 \pm 0.168$ \\
         WLF-SB (ours) &  $0.816 \pm 7.7\text{e-}4$ & $39.240 \pm 0.068$ & $47.788 \pm 0.111$ & $46.177 \pm 0.083$ & $55.716 \pm 0.058$ \\
         $[\text{SF}]^2$ M-Geo \citep{tong2023simulation} & 
         $0.879 \pm 0.148$
         & $38.524 \pm 0.293$ & $\mathbf{44.795 \pm 1.911}$ & $44.498 \pm 0.416$ & $\mathbf{52.203 \pm 1.957}$ \\
         $[\text{SF}]^2$ M-Exact \citep{tong2023simulation} & $\mathbf{0.793 \pm 0.066}$ & $40.009 \pm 0.783$ & $45.337 \pm 2.833$ & $46.530 \pm 0.426$ & $52.888 \pm 1.986$ \\ \hline \\
         WLF-(OT + potential, ours)  & $0.651 \pm 0.002$ & $36.167 \pm 0.031$ & $38.743 \pm 0.060$ & $42.857 \pm 0.045$ & $47.365 \pm 0.051$ \\
         WLF-(UOT + potential, $\lambda = 1$, ours)  & $\mathbf{0.634 \pm 0.001}$ & $\mathbf{34.160 \pm 0.041}$ & $\mathbf{36.131 \pm 0.023}$ & $\mathbf{41.084 \pm 0.043}$ & $\mathbf{45.231 \pm 0.010}$ \\\hline
    \end{tabular}}
    \end{center}
\end{table*}

\begin{example}[\textbf{Unbalanced Optimal Transport}]\label{example:wfr}
The \textit{unbalanced} \textsc{ot} problem arises from the $WFR_\lambda$ geometry, and is useful for modeling mass teleportation and changes in total probability mass when cell birth and death occur as part of the underlying dynamics \citep{schiebinger2019optimal, lu2019accelerating}.  Viewing the dynamical formulation of \textsc{wfr} in \cref{eq:wfr} as a Lagrangian optimization,
\begin{align}
  \hspace*{-.2cm} 
  \frac{1}{2} WFR_\lambda(\mu_0, \mu_{1}&)^2 =~ \inf \limits_{\qm_t \in \marginalcouple} \int_0^1 \cL[\qm_t, \dot{\qm}_t, t] \dt \\
  = \int_0^1 \frac{1}{2} ~&\wfrmetricat{ \dot{\qm}_t, \dot{\qm}_t }{\qm_t}  \dt 
  \;\text{ s.t. }\; \qm_0 = \mu_0, \qm_1 = \mu_1 . \nonumber
\end{align}
\normalsize
Compared to \cref{eq:wfr}, our Lagrangian formulation again optimizes over $\qm_t$ only, and solving the dual requires finding $\lagr_{\dot{\qm}_t}$ such that $\dot{\qm}_t = -\divp{ \qm_t \nabla \lagr_{\dot{\qm}_t}} + \lambda \qm_t \lagr_{\dot{\qm}_t}$ 
as in \cref{eq:wfr_tangent}.
We optimize the objective 
\begin{align}
  \hspace*{-.2cm}  \cS_{uOT} =&~ \inf \limits_{\qm_t \in \marginalcouple} \sup \limits_{\lagr_t} ~ \int \lagr_1 \mu_1 \dx-  \int \lagr_0 \mu_0 \dx
  \\
 &- \int_0^1  \int \left( \frac{\partial \lagr_t}{\partial t}  + \frac{1}{2} \| \nabla \lagr_t \|^2 + \frac{\lambda}{2} \lagr_t^2 \right) \qm_t   \dx \dt , \nonumber
\end{align}
\normalsize
where we recognize the $WFR_{\lambda}$ cotangent metric from \cref{eq:wfr_metric} in the final term,
$\cH[\qm_t, \lagr_t, t] = \cK^*[\qm_t, \lagr_t] = \frac{1}{2}\cometricat{\lagr_t, \lagr_t}{\qm_t}^{WFR_\lambda} = \frac{1}{2} \int \big(\|\nabla \lagr_{t}\|^2  + \lambda ~ \lagr_{t}^2 \big) \qm_t ~\dx$.
\end{example}

\begin{example}[\textbf{Physically-Constrained Optimal Transport}]\label{example:pot}
A popular technique for incorporating inductive bias from biological or geometric prior information into trajectory inference methods is to consider spatial potentials $\cU[\qm_t, t] = \int V_t(x) \qm_t \dx$ \citep{tong2020trajectorynet, koshizuka2022neural, pooladian2023neural}, which are already linear in the density.   In this case, we may consider \textit{any} linearizable kinetic energy  (see \cref{app:dual_kinetic}).
For the $W_2$ transport case, our objective is
\small
\begin{align}
    \cS_{pOT}& =  \inf \limits_{\qm_t\in \marginalcouple} \sup \limits_{\lagr_t} ~ \int \lagr_1 \mu_1 \dx -  \int \lagr_0 \mu_0 \dx
  \\
  &- \int_0^1  \int \left( \frac{\partial \lagr_t (x)}{\partial t}  + \frac{1}{2} \| \nabla \lagr_t(x) \|^2 + \lagrv V_t(x) \right) \qm_t   \dx \dt,  \nonumber
\end{align}
\normalsize
with the optimality conditions
\small
\begin{align}
 \dot{\qm_t} &= - \divp{ \qm_t \nabla \lagr_t }, 
 \quad \ddot{X}_t = \nabla \left[ \frac{\partial \lagr_t }{\partial t}  + \frac{1}{2} \| \nabla \lagr_t \|^2 \right] = - \lagrv \nabla V_t. \nonumber
\end{align}
\normalsize
As in \cref{eq:ot_accel1}, the latter condition
implies that the acceleration is given by the gradient of the spatial potential $V_t(x)$.   We describe 
the 
potentials 
used in our experiments
on scRNA datasets in \cref{sec:experiments}.
\end{example}

\begin{example}[\textbf{Schrödinger Bridge}]\label{example:sb}
For many problems of interest, such as scRNA sequencing \citep{schiebinger2019optimal}, it may be useful to incorporate stochasticity into the dynamics as prior knowledge.  For Brownian-motion diffusion processes with known coefficient $\sigma$, 
the 
dynamical Schrödinger Bridge (\gls{SB}) problem \citep{mikami2008optimal, leonard2013survey, chen2021stochastic} is given by
\begin{align}
\cS_{SB} &= \inf \limits_{\qm_t, v_t}
\int_0^1 \int 
 \frac{1}{2} \| v_t \|^2 \qm_t \dx \dt
\label{eq:sb_std}  \\
 \st \, \dot{\qm}_t &= -\divp{ \qm_t \vel_t } + \sqrttwolagrsbat{} \Delta \qm_t, \,\quad \qm_0 = \mu_0, \quad \qm_1 = \mu_1 . \nonumber
\end{align}
\normalsize



To model the \gls{SB} problem, 
we consider the 
following potential energy with the $W_2$ kinetic energy,
\begin{align}
\cU[\qm_t, t] = -\lagrsbat{} \int \big\|\nabla  \log \qm_t \big\|^2 \qm_t \dx, \quad
\end{align}
\normalsize
which arises from the negative entropy $\cF[\qm_t] 
= \int  (\log \qm_t -1) ~ \qm_t \dx $ via the $W_2$ gradient
$\nabla \frac{\delta}{\delta \qm_t} \cF[\qm_t] = \nabla \log \qm_t$ \citep{figalli2021invitation}.
We assume time-independent $\sigma$ to simplify $\cU[\qm_t, t]$, but consider time-varying $\sigma_t$ in  \cref{example:sb_app}.

To transform the potential energy term into a dual-linearizable form for the \gls{SB} problem, we consider the reparameterization $\hclagr_t =\lagr_t + \sqrttwolagrsbat{} \log \qm_t$, which 
translates between the drift $\nabla \lagr_t$ of the probability flow \textsc{ode} and the drift $\nabla \hclagr_t$ of the Fokker-Planck equation \citep{song2020score}.
With derivations in \cref{app:sb}, the dual objective becomes
\begin{align}
 \hspace*{-.2cm} &\cS_{SB} = \inf \limits_{\qm_t \in \marginalcouple}\sup \limits_{\hclagr_t}  ~ \int \hclagr_1 \mu_1 \dx
  -  \int \hclagr_0 \mu_0 \dx  \label{eq:sb_t_final} \\
  &\qquad -\int_0^1 \int  \Bigg( \frac{\partial \hclagr_t }{\partial t}  +\frac{1}{2} \big\| \nabla \hclagr_t \big\|^2  + \sqrttwolagrsbat{}  \Delta \hclagr_t
\Bigg) ~ \qm_t \dx \dt . \nonumber
\end{align}
\normalsize
\end{example}
\section{Experiments  \protect\footnote{The code reproducing the experiments is available at \href{https://github.com/necludov/wl-mechanics}{https://github.com/necludov/wl-mechanics}}
}\label{sec:experiments}
We apply our methods for trajectory inference of single-cell RNA sequencing data, including the Embryoid body (\textbf{EB}) dataset \citep{moon2019visualizing}, CITE-seq (\textbf{Cite}) and Multiome (\textbf{Multi}) datasets \citep{burkhardt2022multimodal}, and melanoma treatment dataset of \citep{bunne2021learning, pariset2023unbalanced}.

\begin{table}[t]
\vspace*{-5pt}
        \captionof{table}{Results for train/test splits of $5$-dim PCA on EB dataset, with the setting and baseline results taken from \citet[Table 1]{koshizuka2022neural}.   We report W1 distance between test $\mu_{t_i}$ and $\rho_{t_i}$ obtained by running dynamics from $\mu_{t_{i-1}}$.
        }\label{tab:lagr_sb}
    \label{tab:5dim}
    \vspace{10pt}
    \resizebox{0.99\columnwidth}{!}{
    \begin{tabular}{lccccc}
         \multicolumn{1}{c}{\bf Model} & \multicolumn{1}{c}{\bf $t_1$} & \multicolumn{1}{c}{\bf $t_2$} & \multicolumn{1}{c}{\bf $t_3$} & \multicolumn{1}{c}{\bf $t_4$} & Mean
         \\ \hline \\
 Neural SDE \citep{li2020scalable}             & 0.69 & 0.91 & 0.85 & 0.81 & 0.82 \\
 TrajectoryNet \citep{tong2020trajectorynet}        & 0.73 & 1.06 & 0.90 & 1.01 & 0.93 \\
 IPF (GP) \citep{vargas2021solving}               & 0.70 & 1.04 & 0.94 & 0.98 & 0.92 \\
 IPF (NN) \citep{de2021diffusion}     & 0.73 & 0.89 & 0.84 & 0.83 & 0.82 \\
 SB-FBSDE \citep{chen2021likelihood}               & 0.56 & 0.80 & 1.00 & 1.00 & 0.84\\ 
 NLSB \citep{koshizuka2022neural}           & 0.68 & 0.84 & 0.81 & 0.79 & 0.78\\ 
 OT-CFM \citep{tong2023improving}           & 0.78 & 0.76 & 0.77 & 0.75 & 0.77 \\
 \hline \\
 WLF-OT                  & 0.65 & 0.78 & 0.76 & 0.75 & 0.74 \\
 WLF-SB                  & 0.63 & 0.79 & 0.77 & 0.74 & \textbf{0.73} \\
 WLF-(OT + potential)    & 0.64 & 0.77 & 0.76 & 0.76 & \textbf{0.73} \\
 WLF-UOT ($\lambda = 0.1$)               & 0.64 & 0.84 & 0.80 & 0.81 & 0.77 \\
 WLF-(UOT + potential, $\lambda = 0.1$)  & 0.67 & 0.80 & 0.78 & 0.78 & 0.76\\ \hline
    \end{tabular}}
\end{table}

\paragraph{Potential for Physically-Constrained OT}
For all tasks, we consider the simplest possible model of the physical potential accelerating the cells.  For each marginal except the first and the last ones, we estimate the acceleration of its mean using finite differences.
The potential for the corresponding time interval is then $V_t(x) = -\inner{x}{a_t}$, where $a_t$ is the estimated acceleration of the mean value.
For leave-one-out tasks, we include the mean of the left out marginal since the considered data contains too few marginals ($4$ for Cite and Multi) to learn 
a meaningful model of the acceleration.

\vspace*{-.12cm}
\paragraph{Leave-One-Out Marginal Task}
To test the ability of our approaches to approximate interpolating marginal distributions, we follow \cite{tong2020trajectorynet} and evaluate models 
using a leave-one-timepoint-out strategy.   In particular, we train on all marginals except at time $t_i$, and evaluate by
computing
the Wasserstein-$1$ distance between the predicted marginal $\rho_{t_i}$ and the left-out marginal $\mu_{t_i}$. For preprocessing and baselines, we follow \citet{tong2023simulation,tong2023improving} (see App.~\ref{app:sc_experiments_details} for details).

In \cref{tab:manydim}, we report results on EB, Cite, and Multi datasets.
First, we see that our proposed 
WLF-OT method achieves comparable results to related approaches: OT-CFM and I-CFM \citep{tong2023improving}, which use minibatch OT couplings or independent samples of the marginals, respectively.
For OT-CFM, we reproduce the results using a larger model to match the performance of the exact OT solver \citep{flamary2021pot}.  These models represent dynamics with minimal prior knowledge, and thus serve as a baseline when compared against dynamics incorporating additional priors.



Next, we consider Lagrangians which encode various prior information.
WLF-SB (ours) and $[\text{SF}]^2$ M-Geo and -Exact \citep{tong2023simulation}
incorporate stochasticity into the dynamics by solving the \gls{SB} problem; $[\text{SF}]^2$ M-Geo takes advantage of the data manifold geometry by learning from \gls{OT} couplings generated with the approximate geodesic cost; our WLF-UOT incorporates probability mass teleportation using the $WFR$ kinetic energy.
In \cref{tab:manydim}, we 
see that 
WLF-UOT yields
consistent performance improvements across datasets.
Finally, we observe that a good model of the potential function can 
notably
improve 
 performance, using either $W_2$ or $WFR$ kinetic energy.


\begin{table}[t]
\vspace*{-5pt}
\captionof{table}{Results in the setting of \citet[Table 1]{pariset2023unbalanced} (uDSB) for melanoma treatment data with 3 marginals and train/test splits.  We report test MMD and W2 distance between $\mu_1$ and $\rho_1$ obtained by running dynamics from $\mu_0$. }\label{tab:ub_ot}
\vspace{10pt}
\resizebox{0.99\columnwidth}{!}{
\begin{tabular}{lcc}
 \multicolumn{1}{c}{\bf Model} & \multicolumn{1}{c}{MMD} & \multicolumn{1}{c}{$W_2$}
         \\ \hline \\
 \text{\small SB-FBSDE \citep{chen2021likelihood}}          & 1.86e-2 & 6.23 \\
 uDSB  (no growth) \citep{pariset2023unbalanced} & 1.86e-2 & 6.27 \\
 uDSB (w/growth) \citep{pariset2023unbalanced} & 1.75e-2 & 6.11 \\ \hline \\
 WLF-OT (no growth) & \textbf{5.04e-3} & 5.20 \\
 WLF-UOT ($\lambda=0.1$) & 9.16e-3 & \textbf{5.01} \\ \hline
\end{tabular}
}
\end{table}

\textbf{Comparison with SB Baselines on EB Dataset}   To compare against further 
\gls{SB} baselines, we consider the setting of \citet[Table 1]{koshizuka2022neural} on the EB dataset.   Instead of leaving out one marginal, we divide the data using a train/test split and evaluate the W1 distance between the test $\mu_{t_i}$ and $\rho_{t_i}$ obtained by running dynamics from the previous $\mu_{t_{i-1}}$.  In \cref{tab:lagr_sb}, we find that WLF-SB outperforms several \gls{SB} methods from recent literature (see \cref{sec:related}). 

\textbf{Comparison with UOT Baseline on Melanoma Dataset}   To test the ability of our WLF-OT approach to account for cell birth and death, we consider the 50-dim. setting of \citet[Table 1]{pariset2023unbalanced} for melanoma cells undergoing treatment with a cancer drug.   
In \cref{tab:ub_ot}, we show that WLF-OT and WLF-UOT can outperform the unbalanced baseline (uDSB) from \citet{pariset2023unbalanced}.


\section{Related Work}\label{sec:related}

\paragraph{Wasserstein Hamiltonian Flows}
\citet{chow2020wasserstein} develop the notion of a Hamiltonian flow on the Wasserstein manifold and consider several of the same examples discussed here.
While the Hamiltonian and Lagrangian formalisms describe the same integral flow through optimality conditions for $(\qm_t, \dot{\qm}_t)$ and $(\qm_t, \lagr_t)$, \citet{chow2020wasserstein, wu2023parameterized} emphasize solving the Cauchy problem suggested by the Hamiltonian perspective (see also \citet{taghvaei2019accelerated, wang2022accelerated} for accelerated flows). 
Our approach recovers the Hamiltonian flow $(\qm_t, \lagr_t)$ in the cotangent bundle at optimality, but does so by solving a variational problem.
Finally, \citet{conforti2019second, gentil2020dynamical, conforti2018extremal} provide theoretical analysis related to Lagrangian flows on the Wasserstein space.



\paragraph{Flow Matching and Diffusion SB Methods}
Flow Matching methods \citep{liu2022rectified, lipman2022flow, albergo2022building,albergo_stochastic_2023, tong2023improving, tong2023simulation} learn a marginal vector field corresponding to a mixture-of-bridges process parameterized by a coupling and interpolating bridge \citep{shi2023diffusion}.   When samples from the endpoint marginals are coupled via an \gls{OT} plan, Flow Matching solves a dynamical optimal transport problem \citep{pooladian2023multisample}.
Rectified Flow obtains couplings using ODE simulation with the goal of straight-path trajectories for generative modeling \citep{liu2022rectified, liu2022flow}, which is extended to  SDEs in bridge matching methods \citep{ peluchetti2022nondenoising, peluchetti2023diffusion,shi2023diffusion}.  Diffusion Schrödinger Bridge (DSB) methods \citep{de2021diffusion, chen2021likelihood} also update the couplings iteratively based on learned forward and backward \textsc{sde}s, and have recently been adapted to solve the unbalanced \gls{OT} problem in \citet{pariset2023unbalanced}.
Finally, \citet{liu2022deep, liu2023generalized} consider extending DSB or bridge matching methods to solve physically-constrained \gls{SB} problems. 
Unlike the above methods, our approach does not require optimal couplings to sample from the intermediate marginals, and thus avoids both simulating ODEs or SDEs and running minibatch (regularized) \gls{OT} solvers.

\vheader
\vspace*{-.2cm}
\paragraph{Optimal Transport with Lagrangian Cost}
Input-convex neural networks \citep{amos2017input} provide an efficient approach to static \gls{OT} \citep{makkuva2020optimal, korotin2021neural, bunne2021learning, bunne2022supervised} but are limited to the Euclidean cost.
Several works extend to other costs using static \citep{fan2022scalable, pooladian2023neural, uscidda2023monge} or dynamical formulations \citep{liu2021learning, koshizuka2022neural}.
The most general way to define a transport cost is via a Lagrangian action in the state-space (\citet{villani2009optimal} Ch. 7).   While we focus on lifted Lagrangians in the density space, our framework encompasses \textsc{ot} with state-space Lagrangian costs (\cref{app:lagr}).

\section{Conclusion}\label{sec:conclusion}
In this work, we demonstrated that many variations of optimal transport, such as Schrödinger Bridge, unbalanced \gls{OT}, or \gls{OT} with physical constraints can be formulated as Lagrangian action 
minimization on the density manifold. 
We proposed a computational framework for this minimization by deriving a dual objective in terms of cotangent vectors, which correspond to a vector field on the state-space 
and can be parameterized via a neural network.
As an application,
we studied the problem of 
trajectory inference in biological systems, and showed that we can incorporate prior knowledge of the dynamics while respecting marginal constraints on the observed data, resulting in significant improvement in several benchmarks. We expect our approach can be extended to other natural science domains such as quantum mechanics and social sciences by incorporating new prior 
information for learning the underlying dynamics.

\section*{Impact Statement} 
This paper presents work whose goal is to advance the field of machine learning. There are many potential societal consequences of our work, none which we feel must be specifically highlighted here.

\bibliography{icml2024}

\begin{thebibliography}{73}
\providecommand{\natexlab}[1]{#1}
\providecommand{\url}[1]{\texttt{#1}}
\expandafter\ifx\csname urlstyle\endcsname\relax
  \providecommand{\doi}[1]{doi: #1}\else
  \providecommand{\doi}{doi: \begingroup \urlstyle{rm}\Url}\fi

\bibitem[Albergo \& Vanden-Eijnden(2022)Albergo and Vanden-Eijnden]{albergo2022building}
Albergo, M.~S. and Vanden-Eijnden, E.
\newblock Building normalizing flows with stochastic interpolants.
\newblock \emph{International Conference on Learning Representations}, 2022.

\bibitem[Albergo et~al.(2023)Albergo, Boffi, and {Vanden-Eijnden}]{albergo_stochastic_2023}
Albergo, M.~S., Boffi, N.~M., and {Vanden-Eijnden}, E.
\newblock Stochastic interpolants: A unifying framework for flows and diffusions.
\newblock \emph{arXiv preprint 2303.08797}, 2023.

\bibitem[Ambrosio et~al.(2008)Ambrosio, Gigli, and Savar{\'e}]{ambrosio2008gradient}
Ambrosio, L., Gigli, N., and Savar{\'e}, G.
\newblock \emph{Gradient flows: in metric spaces and in the space of probability measures}.
\newblock Springer Science \& Business Media, 2008.

\bibitem[Amos et~al.(2017)Amos, Xu, and Kolter]{amos2017input}
Amos, B., Xu, L., and Kolter, J.~Z.
\newblock {Input convex neural networks}.
\newblock In \emph{International Conference on Machine Learning}, pp.\  146--155. PMLR, 2017.

\bibitem[Arnol'd(2013)]{arnol2013mathematical}
Arnol'd, V.~I.
\newblock \emph{Mathematical methods of classical mechanics}, volume~60.
\newblock Springer Science \& Business Media, 2013.

\bibitem[Bauer et~al.(2016)Bauer, Bruveris, and Michor]{bauer2016uniqueness}
Bauer, M., Bruveris, M., and Michor, P.~W.
\newblock Uniqueness of the fisher--rao metric on the space of smooth densities.
\newblock \emph{Bulletin of the London Mathematical Society}, 48\penalty0 (3):\penalty0 499--506, 2016.

\bibitem[Benamou \& Brenier(2000)Benamou and Brenier]{benamou2000computational}
Benamou, J.-D. and Brenier, Y.
\newblock A computational fluid mechanics solution to the monge-kantorovich mass transfer problem.
\newblock \emph{Numerische Mathematik}, 84\penalty0 (3):\penalty0 375--393, 2000.

\bibitem[Benamou et~al.(2019)Benamou, Gallou{\"e}t, and Vialard]{benamou2019second}
Benamou, J.-D., Gallou{\"e}t, T.~O., and Vialard, F.-X.
\newblock Second-order models for optimal transport and cubic splines on the wasserstein space.
\newblock \emph{Foundations of Computational Mathematics}, 19:\penalty0 1113--1143, 2019.

\bibitem[Bunne et~al.(2021)Bunne, Stark, Gut, del Castillo, Lehmann, Pelkmans, Krause, and R{\"a}tsch]{bunne2021learning}
Bunne, C., Stark, S.~G., Gut, G., del Castillo, J.~S., Lehmann, K.-V., Pelkmans, L., Krause, A., and R{\"a}tsch, G.
\newblock Learning single-cell perturbation responses using neural optimal transport.
\newblock \emph{bioRxiv}, pp.\  2021--12, 2021.

\bibitem[Bunne et~al.(2022)Bunne, Krause, and Cuturi]{bunne2022supervised}
Bunne, C., Krause, A., and Cuturi, M.
\newblock Supervised training of conditional monge maps.
\newblock \emph{Advances in Neural Information Processing Systems}, 35:\penalty0 6859--6872, 2022.

\bibitem[Burkhardt et~al.(2022)Burkhardt, Bloom, Cannoodt, Luecken, Krishnaswamy, Lance, Pisco, and Theis]{burkhardt2022multimodal}
Burkhardt, D., Bloom, J., Cannoodt, R., Luecken, M.~D., Krishnaswamy, S., Lance, C., Pisco, A.~O., and Theis, F.~J.
\newblock Multimodal single-cell integration across time, individuals, and batches.
\newblock In \emph{{{NeurIPS Competitions}}}, 2022.

\bibitem[Chen et~al.(2021{\natexlab{a}})Chen, Liu, and Theodorou]{chen2021likelihood}
Chen, T., Liu, G.-H., and Theodorou, E.
\newblock Likelihood training of schr{\"o}dinger bridge using forward-backward sdes theory.
\newblock In \emph{International Conference on Learning Representations}, 2021{\natexlab{a}}.

\bibitem[Chen et~al.(2023)Chen, Liu, Tao, and Theodorou]{chen2023deep}
Chen, T., Liu, G.-H., Tao, M., and Theodorou, E.~A.
\newblock Deep momentum multi-marginal {Schrödinger} bridge.
\newblock \emph{arXiv preprint arXiv:2303.01751}, 2023.

\bibitem[Chen et~al.(2018)Chen, Conforti, and Georgiou]{chen2018measure}
Chen, Y., Conforti, G., and Georgiou, T.~T.
\newblock Measure-valued spline curves: An optimal transport viewpoint.
\newblock \emph{SIAM Journal on Mathematical Analysis}, 50\penalty0 (6):\penalty0 5947--5968, 2018.

\bibitem[Chen et~al.(2021{\natexlab{b}})Chen, Georgiou, and Pavon]{chen2021stochastic}
Chen, Y., Georgiou, T.~T., and Pavon, M.
\newblock Stochastic control liaisons: Richard sinkhorn meets gaspard monge on a schrodinger bridge.
\newblock \emph{Siam Review}, 63\penalty0 (2):\penalty0 249--313, 2021{\natexlab{b}}.

\bibitem[Chewi et~al.(2021)Chewi, Clancy, Le~Gouic, Rigollet, Stepaniants, and Stromme]{chewi2021fast}
Chewi, S., Clancy, J., Le~Gouic, T., Rigollet, P., Stepaniants, G., and Stromme, A.
\newblock Fast and smooth interpolation on wasserstein space.
\newblock In \emph{International Conference on Artificial Intelligence and Statistics}, pp.\  3061--3069. PMLR, 2021.

\bibitem[Chizat et~al.(2018{\natexlab{a}})Chizat, Peyr{\'e}, Schmitzer, and Vialard]{chizat2018interpolating}
Chizat, L., Peyr{\'e}, G., Schmitzer, B., and Vialard, F.-X.
\newblock An interpolating distance between optimal transport and fisher--rao metrics.
\newblock \emph{Foundations of Computational Mathematics}, 18\penalty0 (1):\penalty0 1--44, 2018{\natexlab{a}}.

\bibitem[Chizat et~al.(2018{\natexlab{b}})Chizat, Peyr{\'e}, Schmitzer, and Vialard]{chizat2018unbalanced}
Chizat, L., Peyr{\'e}, G., Schmitzer, B., and Vialard, F.-X.
\newblock Unbalanced optimal transport: Dynamic and {Kantorovich} formulations.
\newblock \emph{Journal of Functional Analysis}, 274\penalty0 (11):\penalty0 3090--3123, 2018{\natexlab{b}}.

\bibitem[Chow et~al.(2020)Chow, Li, and Zhou]{chow2020wasserstein}
Chow, S.-N., Li, W., and Zhou, H.
\newblock Wasserstein hamiltonian flows.
\newblock \emph{Journal of Differential Equations}, 268\penalty0 (3):\penalty0 1205--1219, 2020.

\bibitem[Clancy \& Suarez(2022)Clancy and Suarez]{clancy2022wasserstein}
Clancy, J. and Suarez, F.
\newblock Wasserstein-fisher-rao splines.
\newblock \emph{arXiv preprint arXiv:2203.15728}, 2022.

\bibitem[Conforti(2019)]{conforti2019second}
Conforti, G.
\newblock A second order equation for schr{\"o}dinger bridges with applications to the hot gas experiment and entropic transportation cost.
\newblock \emph{Probability Theory and Related Fields}, 174\penalty0 (1):\penalty0 1--47, 2019.

\bibitem[Conforti \& Pavon(2018)Conforti and Pavon]{conforti2018extremal}
Conforti, G. and Pavon, M.
\newblock Extremal flows in wasserstein space.
\newblock \emph{Journal of mathematical physics}, 59\penalty0 (6), 2018.

\bibitem[De~Bortoli et~al.(2021)De~Bortoli, Thornton, Heng, and Doucet]{de2021diffusion}
De~Bortoli, V., Thornton, J., Heng, J., and Doucet, A.
\newblock Diffusion schr{\"o}dinger bridge with applications to score-based generative modeling.
\newblock \emph{Advances in Neural Information Processing Systems}, 34:\penalty0 17695--17709, 2021.

\bibitem[Fan et~al.(2022)Fan, Liu, Ma, Chen, and Zhou]{fan2022scalable}
Fan, J., Liu, S., Ma, S., Chen, Y., and Zhou, H.-M.
\newblock Scalable computation of monge maps with general costs.
\newblock In \emph{ICLR Workshop on Deep Generative Models for Highly Structured Data}, 2022.

\bibitem[Figalli \& Glaudo(2021)Figalli and Glaudo]{figalli2021invitation}
Figalli, A. and Glaudo, F.
\newblock \emph{An invitation to optimal transport, Wasserstein distances, and gradient flows}.
\newblock 2021.

\bibitem[Flamary et~al.(2021)Flamary, Courty, Gramfort, Alaya, Boisbunon, Chambon, Chapel, Corenflos, Fatras, Fournier, et~al.]{flamary2021pot}
Flamary, R., Courty, N., Gramfort, A., Alaya, M.~Z., Boisbunon, A., Chambon, S., Chapel, L., Corenflos, A., Fatras, K., Fournier, N., et~al.
\newblock Pot: Python optimal transport.
\newblock \emph{The Journal of Machine Learning Research}, 22\penalty0 (1):\penalty0 3571--3578, 2021.

\bibitem[Gentil et~al.(2020)Gentil, L{\'e}onard, and Ripani]{gentil2020dynamical}
Gentil, I., L{\'e}onard, C., and Ripani, L.
\newblock Dynamical aspects of the generalized schr{\"o}dinger problem via otto calculus--a heuristic point of view.
\newblock \emph{Revista matem{\'a}tica iberoamericana}, 36\penalty0 (4):\penalty0 1071--1112, 2020.

\bibitem[Hashimoto et~al.(2016)Hashimoto, Gifford, and Jaakkola]{hashimoto2016learning}
Hashimoto, T., Gifford, D., and Jaakkola, T.
\newblock Learning population-level diffusions with generative rnns.
\newblock In \emph{International Conference on Machine Learning}, pp.\  2417--2426. PMLR, 2016.

\bibitem[Kondratyev et~al.(2016)Kondratyev, Monsaingeon, and Vorotnikov]{kondratyev2016new}
Kondratyev, S., Monsaingeon, L., and Vorotnikov, D.
\newblock A new optimal transport distance on the space of finite radon measures.
\newblock \emph{Advances in Differential Equations}, 21\penalty0 (11/12):\penalty0 1117--1164, 2016.

\bibitem[Korotin et~al.(2021)Korotin, Li, Genevay, Solomon, Filippov, and Burnaev]{korotin2021neural}
Korotin, A., Li, L., Genevay, A., Solomon, J.~M., Filippov, A., and Burnaev, E.
\newblock Do neural optimal transport solvers work? a continuous wasserstein-2 benchmark.
\newblock \emph{Advances in Neural Information Processing Systems}, 34:\penalty0 14593--14605, 2021.

\bibitem[Koshizuka \& Sato(2022)Koshizuka and Sato]{koshizuka2022neural}
Koshizuka, T. and Sato, I.
\newblock Neural lagrangian {Schrödinger} bridge: Diffusion modeling for population dynamics.
\newblock In \emph{The Eleventh International Conference on Learning Representations}, 2022.

\bibitem[Lavenant et~al.(2021)Lavenant, Zhang, Kim, and Schiebinger]{lavenant2021towards}
Lavenant, H., Zhang, S., Kim, Y.-H., and Schiebinger, G.
\newblock Towards a mathematical theory of trajectory inference.
\newblock \emph{arXiv preprint arXiv:2102.09204}, 2021.

\bibitem[L{\'e}ger \& Li(2021)L{\'e}ger and Li]{leger2021hopf}
L{\'e}ger, F. and Li, W.
\newblock Hopf--cole transformation via generalized schr{\"o}dinger bridge problem.
\newblock \emph{Journal of Differential Equations}, 274:\penalty0 788--827, 2021.

\bibitem[L{\'e}onard(2013)]{leonard2013survey}
L{\'e}onard, C.
\newblock A survey of the {Schrödinger} problem and some of its connections with optimal transport.
\newblock \emph{arXiv preprint arXiv:1308.0215}, 2013.

\bibitem[Li et~al.(2020)Li, Wong, Chen, and Duvenaud]{li2020scalable}
Li, X., Wong, T.-K.~L., Chen, R.~T., and Duvenaud, D.
\newblock Scalable gradients for stochastic differential equations.
\newblock In \emph{International Conference on Artificial Intelligence and Statistics}, pp.\  3870--3882. PMLR, 2020.

\bibitem[Liero et~al.(2016)Liero, Mielke, and Savar{\'e}]{liero2016optimal}
Liero, M., Mielke, A., and Savar{\'e}, G.
\newblock Optimal transport in competition with reaction: The {Hellinger--Kantorovich} distance and geodesic curves.
\newblock \emph{SIAM Journal on Mathematical Analysis}, 48\penalty0 (4):\penalty0 2869--2911, 2016.

\bibitem[Liero et~al.(2018)Liero, Mielke, and Savar{\'e}]{liero2018optimal}
Liero, M., Mielke, A., and Savar{\'e}, G.
\newblock Optimal entropy-transport problems and a new {Hellinger--Kantorovich} distance between positive measures.
\newblock \emph{Inventiones mathematicae}, 211\penalty0 (3):\penalty0 969--1117, 2018.

\bibitem[Lipman et~al.(2022)Lipman, Chen, Ben-Hamu, Nickel, and Le]{lipman2022flow}
Lipman, Y., Chen, R.~T., Ben-Hamu, H., Nickel, M., and Le, M.
\newblock Flow matching for generative modeling.
\newblock \emph{International Conference on Learning Representations}, 2022.

\bibitem[Liu et~al.(2023)Liu, Lipman, Nickel, Karrer, Theodorou, and Chen]{liu2023generalized}
Liu, G., Lipman, Y., Nickel, M., Karrer, B., Theodorou, E.~A., and Chen, R.~T.
\newblock Generalized schrödinger bridge matching, 9 2023.

\bibitem[Liu et~al.(2022{\natexlab{a}})Liu, Chen, So, and Theodorou]{liu2022deep}
Liu, G.-H., Chen, T., So, O., and Theodorou, E.
\newblock Deep generalized schr{\"o}dinger bridge.
\newblock \emph{Advances in Neural Information Processing Systems}, 35:\penalty0 9374--9388, 2022{\natexlab{a}}.

\bibitem[Liu(2022)]{liu2022rectified}
Liu, Q.
\newblock Rectified flow: A marginal preserving approach to optimal transport.
\newblock \emph{arXiv preprint arXiv:2209.14577}, 2022.

\bibitem[Liu et~al.(2021)Liu, Ma, Chen, Zha, and Zhou]{liu2021learning}
Liu, S., Ma, S., Chen, Y., Zha, H., and Zhou, H.
\newblock Learning high dimensional wasserstein geodesics.
\newblock \emph{arXiv preprint arXiv:2102.02992}, 2021.

\bibitem[Liu et~al.(2022{\natexlab{b}})Liu, Gong, and Liu]{liu2022flow}
Liu, X., Gong, C., and Liu, Q.
\newblock Flow straight and fast: Learning to generate and transfer data with rectified flow.
\newblock \emph{International Conference on Learning Representations}, 2022{\natexlab{b}}.

\bibitem[Loshchilov \& Hutter(2017)Loshchilov and Hutter]{loshchilov2017decoupled}
Loshchilov, I. and Hutter, F.
\newblock Decoupled weight decay regularization.
\newblock \emph{arXiv preprint arXiv:1711.05101}, 2017.

\bibitem[Lu et~al.(2019)Lu, Lu, and Nolen]{lu2019accelerating}
Lu, Y., Lu, J., and Nolen, J.
\newblock Accelerating langevin sampling with birth-death.
\newblock \emph{arXiv preprint arXiv:1905.09863}, 2019.

\bibitem[Macosko et~al.(2015)Macosko, Basu, Satija, Nemesh, Shekhar, Goldman, Tirosh, Bialas, Kamitaki, Martersteck, et~al.]{macosko2015highly}
Macosko, E.~Z., Basu, A., Satija, R., Nemesh, J., Shekhar, K., Goldman, M., Tirosh, I., Bialas, A.~R., Kamitaki, N., Martersteck, E.~M., et~al.
\newblock Highly parallel genome-wide expression profiling of individual cells using nanoliter droplets.
\newblock \emph{Cell}, 161\penalty0 (5):\penalty0 1202--1214, 2015.

\bibitem[Makkuva et~al.(2020)Makkuva, Taghvaei, Oh, and Lee]{makkuva2020optimal}
Makkuva, A., Taghvaei, A., Oh, S., and Lee, J.
\newblock Optimal transport mapping via input convex neural networks.
\newblock In \emph{International Conference on Machine Learning}, pp.\  6672--6681. PMLR, 2020.

\bibitem[Mikami(2008)]{mikami2008optimal}
Mikami, T.
\newblock Optimal transportation problem as stochastic mechanics.
\newblock \emph{Selected Papers on Probability and Statistics, Amer. Math. Soc. Transl. Ser}, 2\penalty0 (227):\penalty0 75--94, 2008.

\bibitem[Moon et~al.(2019)Moon, van Dijk, Wang, Gigante, Burkhardt, Chen, Yim, Elzen, Hirn, Coifman, et~al.]{moon2019visualizing}
Moon, K.~R., van Dijk, D., Wang, Z., Gigante, S., Burkhardt, D.~B., Chen, W.~S., Yim, K., Elzen, A. v.~d., Hirn, M.~J., Coifman, R.~R., et~al.
\newblock Visualizing structure and transitions in high-dimensional biological data.
\newblock \emph{Nature biotechnology}, 37\penalty0 (12):\penalty0 1482--1492, 2019.

\bibitem[Neklyudov et~al.(2023)Neklyudov, Brekelmans, Severo, and Makhzani]{neklyudov2023action}
Neklyudov, K., Brekelmans, R., Severo, D., and Makhzani, A.
\newblock Action matching: Learning stochastic dynamics from samples.
\newblock In \emph{International Conference on Machine Learning}, 2023.

\bibitem[Otto(2001)]{otto2001geometry}
Otto, F.
\newblock The geometry of dissipative evolution equations: the porous medium equation.
\newblock \emph{Comm. Partial Differential Equations}, 26:\penalty0 101--174, 2001.

\bibitem[Pariset et~al.(2023)Pariset, Hsieh, Bunne, Krause, and De~Bortoli]{pariset2023unbalanced}
Pariset, M., Hsieh, Y.-P., Bunne, C., Krause, A., and De~Bortoli, V.
\newblock Unbalanced diffusion {Schrödinger} bridge.
\newblock \emph{arXiv preprint arXiv:2306.09099}, 2023.

\bibitem[Peluchetti(2022)]{peluchetti2022nondenoising}
Peluchetti, S.
\newblock Non-denoising forward-time diffusions, 2022.
\newblock URL \url{https://openreview.net/forum?id=oVfIKuhqfC}.

\bibitem[Peluchetti(2023)]{peluchetti2023diffusion}
Peluchetti, S.
\newblock Diffusion bridge mixture transports, {Schrödinger} bridge problems and generative modeling.
\newblock \emph{arXiv preprint arXiv:2304.00917}, 2023.

\bibitem[Pfau et~al.(2020)Pfau, Spencer, Matthews, and Foulkes]{pfau2020ab}
Pfau, D., Spencer, J.~S., Matthews, A.~G., and Foulkes, W. M.~C.
\newblock Ab initio solution of the many-electron schr{\"o}dinger equation with deep neural networks.
\newblock \emph{Physical Review Research}, 2\penalty0 (3):\penalty0 033429, 2020.

\bibitem[Pooladian et~al.(2023{\natexlab{a}})Pooladian, Ben-Hamu, Domingo-Enrich, Amos, Lipman, and Chen]{pooladian2023multisample}
Pooladian, A.-A., Ben-Hamu, H., Domingo-Enrich, C., Amos, B., Lipman, Y., and Chen, R.
\newblock Multisample flow matching: Straightening flows with minibatch couplings.
\newblock \emph{International Conference on Machine Learning}, 2023{\natexlab{a}}.

\bibitem[Pooladian et~al.(2023{\natexlab{b}})Pooladian, Domingo-Enrich, Chen, and Amos]{pooladian2023neural}
Pooladian, A.-A., Domingo-Enrich, C., Chen, R.~T., and Amos, B.
\newblock Neural optimal transport with lagrangian costs.
\newblock In \emph{ICML Workshop on New Frontiers in Learning, Control, and Dynamical Systems}, 2023{\natexlab{b}}.

\bibitem[Ronneberger et~al.(2015)Ronneberger, Fischer, and Brox]{ronneberger2015u}
Ronneberger, O., Fischer, P., and Brox, T.
\newblock U-net: Convolutional networks for biomedical image segmentation.
\newblock In \emph{Medical Image Computing and Computer-Assisted Intervention--MICCAI 2015: 18th International Conference, Munich, Germany, October 5-9, 2015, Proceedings, Part III 18}, pp.\  234--241. Springer, 2015.

\bibitem[Saelens et~al.(2019)Saelens, Cannoodt, Todorov, and Saeys]{saelens2019comparison}
Saelens, W., Cannoodt, R., Todorov, H., and Saeys, Y.
\newblock A comparison of single-cell trajectory inference methods.
\newblock \emph{Nature biotechnology}, 37\penalty0 (5):\penalty0 547--554, 2019.

\bibitem[Schachter(2017)]{schachter2017eulerian}
Schachter, B.
\newblock \emph{An Eulerian Approach to Optimal Transport with Applications to the Otto Calculus}.
\newblock University of Toronto (Canada), 2017.

\bibitem[Schiebinger(2021)]{schiebinger2021reconstructing}
Schiebinger, G.
\newblock Reconstructing developmental landscapes and trajectories from single-cell data.
\newblock \emph{Current Opinion in Systems Biology}, 27:\penalty0 100351, 2021.

\bibitem[Schiebinger et~al.(2019)Schiebinger, Shu, Tabaka, Cleary, Subramanian, Solomon, Gould, Liu, Lin, Berube, et~al.]{schiebinger2019optimal}
Schiebinger, G., Shu, J., Tabaka, M., Cleary, B., Subramanian, V., Solomon, A., Gould, J., Liu, S., Lin, S., Berube, P., et~al.
\newblock Optimal-transport analysis of single-cell gene expression identifies developmental trajectories in reprogramming.
\newblock \emph{Cell}, 176\penalty0 (4):\penalty0 928--943, 2019.

\bibitem[Shi et~al.(2023)Shi, De~Bortoli, Campbell, and Doucet]{shi2023diffusion}
Shi, Y., De~Bortoli, V., Campbell, A., and Doucet, A.
\newblock Diffusion {Schrödinger} bridge matching.
\newblock \emph{arXiv preprint arXiv:2303.16852}, 2023.

\bibitem[Song et~al.(2020)Song, Sohl-Dickstein, Kingma, Kumar, Ermon, and Poole]{song2020score}
Song, Y., Sohl-Dickstein, J., Kingma, D.~P., Kumar, A., Ermon, S., and Poole, B.
\newblock Score-based generative modeling through stochastic differential equations.
\newblock In \emph{International Conference on Learning Representations}, 2020.

\bibitem[Taghvaei \& Mehta(2019)Taghvaei and Mehta]{taghvaei2019accelerated}
Taghvaei, A. and Mehta, P.
\newblock Accelerated flow for probability distributions.
\newblock In \emph{International conference on machine learning}, pp.\  6076--6085. PMLR, 2019.

\bibitem[Tong et~al.(2020)Tong, Huang, Wolf, Van~Dijk, and Krishnaswamy]{tong2020trajectorynet}
Tong, A., Huang, J., Wolf, G., Van~Dijk, D., and Krishnaswamy, S.
\newblock Trajectorynet: A dynamic optimal transport network for modeling cellular dynamics.
\newblock In \emph{International conference on machine learning}, pp.\  9526--9536. PMLR, 2020.

\bibitem[Tong et~al.(2023{\natexlab{a}})Tong, Malkin, Fatras, Atanackovic, Zhang, Huguet, Wolf, and Bengio]{tong2023simulation}
Tong, A., Malkin, N., Fatras, K., Atanackovic, L., Zhang, Y., Huguet, G., Wolf, G., and Bengio, Y.
\newblock Simulation-free schr\"odinger bridges via score and flow matching.
\newblock \emph{arXiv preprint arXiv:2307.03672}, 2023{\natexlab{a}}.

\bibitem[Tong et~al.(2023{\natexlab{b}})Tong, Malkin, Huguet, Zhang, Rector-Brooks, Fatras, Wolf, and Bengio]{tong2023improving}
Tong, A., Malkin, N., Huguet, G., Zhang, Y., Rector-Brooks, J., Fatras, K., Wolf, G., and Bengio, Y.
\newblock Improving and generalizing flow-based generative models with minibatch optimal transport.
\newblock In \emph{ICML Workshop on New Frontiers in Learning, Control, and Dynamical Systems}, 2023{\natexlab{b}}.

\bibitem[Uscidda \& Cuturi(2023)Uscidda and Cuturi]{uscidda2023monge}
Uscidda, T. and Cuturi, M.
\newblock The monge gap: A regularizer to learn all transport maps.
\newblock \emph{arXiv preprint arXiv:2302.04953}, 2023.

\bibitem[Vargas et~al.(2021)Vargas, Thodoroff, Lamacraft, and Lawrence]{vargas2021solving}
Vargas, F., Thodoroff, P., Lamacraft, A., and Lawrence, N.
\newblock Solving schr{\"o}dinger bridges via maximum likelihood.
\newblock \emph{Entropy}, 23\penalty0 (9):\penalty0 1134, 2021.

\bibitem[Villani(2009)]{villani2009optimal}
Villani, C.
\newblock \emph{Optimal transport: old and new}, volume 338.
\newblock Springer, 2009.

\bibitem[Wang \& Li(2022)Wang and Li]{wang2022accelerated}
Wang, Y. and Li, W.
\newblock Accelerated information gradient flow.
\newblock \emph{Journal of Scientific Computing}, 90:\penalty0 1--47, 2022.

\bibitem[Wu et~al.(2023)Wu, Liu, Ye, and Zhou]{wu2023parameterized}
Wu, H., Liu, S., Ye, X., and Zhou, H.
\newblock Parameterized wasserstein hamiltonian flow.
\newblock \emph{arXiv preprint arXiv:2306.00191}, 2023.

\end{thebibliography}
\bibliographystyle{icml2024}

\newpage
\appendix
\onecolumn

\section{General Dual Objectives for Wasserstein Lagrangian Flows}\label{app:dual}
In this section, we derive the general forms for the Hamiltonian dual objectives arising from Wasserstein Lagrangian Flows.    We prove \cref{prop:dual_obj} and derive the general dual objective in \cref{eq:dual_intermediate} of the main text, before considering the effect of multiple marginal constraints in \cref{app:multiple-marginal}.
We defer explicit calculation of Hamiltonians for important special cases to \cref{app:additional_examples}.
\wlfdual*
Recall the definition of the Legendre transform for $\cL[\qm_t,\dot{\qm}_t, t]$ strictly convex in $\dot{\qm}_t$,
\begin{align}
\cH[\qm_t, \lagr_t, t]  &= \sup \limits_{\dot{\qm}_t \in \cT_{\qm_t}\cP} ~ \int \lagr_t \dot{\qm_t} \dx - \mathcal{L}[\qm_t, \dot{\qm}_t,t] \label{appeq:hamiltonian} \qquad \\
\mathcal{L}[\qm_t, \dot{\qm}_t,t]  &= \sup \limits_{\lagr_t \in \cT^*_{\qm_t}\cP} ~ \int \lagr_t \dot{\qm_t} \dx -  \cH[\qm_t, \lagr_t, t]\label{eq:lagr_hamil}
\end{align}
\begin{proof}

We prove the case of $M=2$ here and the case of $M>2$ below in \cref{app:multiple-marginal}.  

Denote the set of curves of marginal densities $\qm_t$ with the prescribed endpoint marginals as $\marginalcouple = \{ \qm_t | ~\qm_t \in \cP(\cX)~ \forall t,  \qm_0 = \mu_0, \qm_1  = \mu_1 \} $.   The result follows directly from the definition of the Legendre transform in \cref{appeq:hamiltonian} and integration by parts in time in step $(i)$,
\begin{align}
\cS_{\cL}( \{\mu_{0,1}\}) &= \inf \limits_{\qm_t} \int_0^1 \mathcal{L}[\qm_t, \dot{\qm}_t,t]  \dt
\st \quad  \qm_0 = \mu_0, \qquad \qm_1 = \mu_1 \label{eq:pf_block} \\
&= \inf \limits_{\qm_t \in \marginalcouple} \int_0^1 \mathcal{L}[\qm_t, \dot{\qm}_t,t] \dt \nonumber  \\
&= \inf \limits_{\qm_t \in \marginalcouple} \sup \limits_{\lagr_t \in \cT^*_{\qm_t}\cP} ~ \int_0^1 \left( \int \lagr_t \dot{\qm_t} \dx  -  \cH[\qm_t, \lagr_t, t] \right)\dt \nonumber  \\
&\overset{(i)}{=}  \inf \limits_{\qm_t \in \marginalcouple} \sup \limits_{\lagr_t } \int \lagr_1 \qm_1  \dx - \int \lagr_0 \qm_0 \dx - \int_0^1 \left( \int \frac{ \partial \lagr_t}{\partial t}\qm_t \dx +  \cH[\qm_t, \lagr_t, t] \right)\dt \nonumber  \\
&\overset{(ii)}{=}  \inf \limits_{\qm_t  \in \marginalcouple} \sup \limits_{\lagr_t } \int \lagr_1 \mu_1  \dx - \int \lagr_0 \mu_0 \dx - \int_0^1 \left( \int \frac{ \partial \lagr_t}{\partial t}\qm_t \dx +  \cH[\qm_t, \lagr_t, t] \right)\dt \nonumber
\end{align}
which is the desired result.
In (ii), we use the fact that $\qm_0 = \mu_0$, $\qm_1=\mu_1$ for $\qm_t \in \marginalcouple$.  
Finally, note that $\lagr_t \in \cT^*_{\qm_t}\cP$ simply identifies $\lagr_t$ as a cotangent vector and does not impose meaningful constraints on the form of $\lagr_t \in \cC^\infty(\cX)$, so we drop this from the optimization in step (i).  
\end{proof}

\subsection{Multiple Marginal Constraints}\label{app:multiple-marginal}
Consider multiple marginal constraints in the Lagrangian action minimization problem for $\mathcal{L}[\qm_t, \dot{\qm}_t,t]$  strictly convex in $\dot{\qm}_t$,
\begin{align}
\cS_\cL(\{\mu_{t_i}\}_{i=0}^{M-1}) &= \inf \limits_{\qm_t} \int_0^1 \mathcal{L}[\qm_t, \dot{\qm}_t,t] \dt
\, \st \,\, \qm_{t_i}   = \mu_{t_i} \,\, \text{\small $(\forall \,\, 0\leq i \leq M-1)$}  \\
&= \inf \limits_{\qm_t \in \mmarginalcouple} \int_0^1 \mathcal{L}[\qm_t, \dot{\qm}_t,t] \dt \nonumber
\end{align}
As in the proof of \cref{prop:dual_obj}, the dual becomes
\begin{align}
\cS_\cL(\{\mu_{t_i}\}_{i=0}^{M-1}) 
&= \inf \limits_{\qm_t \in \mmarginalcouple} \sup \limits_{\lagr_t \in \cT^*_{\qm_t}\cP} ~ \int_0^1 \left( \int \lagr_t \dot{\qm_t} \dx  -  \cH[\qm_t, \lagr_t, t] \right)\dt \nonumber  \\
&=  \inf \limits_{\qm_t \in \mmarginalcouple} \sup \limits_{\lagr_t } \int \lagr_1 \qm_1  \dx - \int \lagr_0 \qm_0 \dx - \int_0^1 \left( \int \frac{ \partial \lagr_t}{\partial t}\qm_t \dx +  \cH[\qm_t, \lagr_t, t] \right)\dt \nonumber  \\
&=  \inf \limits_{\qm_t  \in \mmarginalcouple} \sup \limits_{\lagr_t } \int \lagr_1 \mu_1  \dx - \int \lagr_0 \mu_0 \dx - \int_0^1 \left( \int \frac{ \partial \lagr_t}{\partial t}\qm_t \dx +  \cH[\qm_t, \lagr_t, t] \right)\dt \label{eq:last_multi-marginal}
\end{align}
where the intermediate marginal constraints do not affect the result.
Crucially, as discussed in \cref{sec:parametrization}, our sampling approach satisfies the marginal constraints $\qm_{t_i}(x_{t_i}) = \mu_{t_i}(x_{t_i}) $ by design.

\paragraph{Piecewise Lagrangian Optimization} 
We confirm that $\cS_{\cL}(\{\mu_{t_i}\}_{i=0}^{M-1}) = \sum_{i=0}^{M-2} \cS_{\cL}(\{\mu_{t_{i},t_{i+1}}\})$ matches the sum of piecewise optimizations, which is a result of the fact that the action functional $\cA_{\cL}[\qm_t]$ is simply an integration over time.
To confirm this, consider the piecewise dual objective for action-minimization problems between $\{\mu_{0,t_1}\}$ and $\{\mu_{t_1,1}\}$ ($M=3$),
\small
\begin{align}
    \cS_{\cL}(\{\mu_{0,t_1}\}) + \cS_{\cL}(\{\mu_{t_1,1}\}) 
    &= \inf \limits_{\qm_t \in \marginalcoupleof{\mu_0,\mu_{t_1}}} \sup \limits_{\lagr_t} ~ \int \lagr_{t_1} \mu_{t_1} \dx -  \int \lagr_0 \mu_0 \dx +\int_0^{t_1}   \left( \int \frac{\partial \lagr_t}{\partial t} \qm_t \dx+ \cH[\qm_t, \lagr_t, t] \right) \dt \label{eq:piecewise} \\ 
    &\phantom{===} +  \inf \limits_{\qm_t \in \marginalcoupleof{\mu_{t_1}, \mu_1}} \sup \limits_{\lagr_t} ~ \int \lagr_{1} \mu_{1}\dx  -  \int \lagr_{t_1} \mu_{t_1} \dx +\int_{t_1}^1   \left( \int \frac{\partial \lagr_t}{\partial t} \qm_t \dx+ \cH[\qm_t, \lagr_t, t] \right) \dt \nonumber \\ 
    &=  \inf \limits_{\qm_t  \in \mmarginalcouple} \sup \limits_{\lagr_t } \int \lagr_1 \mu_1  \dx - \int \lagr_0 \mu_0 \dx - \int_0^1 \left( \int \frac{ \partial \lagr_t}{\partial t}\qm_t \dx +  \cH[\qm_t, \lagr_t, t] \right)\dt
\end{align}
\normalsize
After telescoping cancellation and taking the union of the constraints, we obtain the same objective as in \cref{eq:last_multi-marginal}.

\section{Tractable Objectives for Special Cases}\label{app:additional_examples}
In this section, we calculate Hamiltonians and explicit dual objectives for important special cases of Wasserstein Lagrangian Flows, including those in \cref{sec:examples}.  

We consider several important kinetic energies in  \cref{app:dual_kinetic}, including the $W_2$ and $WFR_\lambda$ metrics (\cref{app:wfr}) and the case of \gls{OT} costs defined by general ground-space Lagrangians (\cref{app:lagr}).
In \cref{app:sb}, we provide further derivations to obtain a linear dual objective for the Schrödinger Bridge problem.   Finally, we highlight the lack of dual linearizability for the case of the 
Schrödinger Equation \cref{app:other_examples} \cref{example:se}. 

\subsection{Dual Kinetic Energy from $W_2$, $WFR$, or Ground-Space Lagrangian Costs}\label{app:dual_kinetic}
\cref{prop:dual_obj} makes progress toward a dual objective \textit{without} considering the continuity equation or dynamics in the ground space, by instead invoking the Legendre transform $\cH[\qm_t, \lagr_t, t]$ of a given Lagrangian $\mathcal{L}[\qm_t, \dot{\qm}_t,t]$ which is strictly convex in $\dot{\qm}_t$.
However, to derive $\cH[\qm_t, \lagr_t, t]$ and optimize objectives of the form \cref{eq:dual_intermediate}, we will need to represent the tangent vector on the space of densities $\dot{\qm}_t$, for example using a vector field $v_t$ and growth term $g_t$ as in \cref{eq:wfr}.

Given a Lagrangian $\mathcal{L}[\qm_t, \dot{\qm}_t,t]$, we seek to solve the optimization 
\begin{align}
\cH[\qm_t, \lagr_t, t]  &= \sup \limits_{\dot{\qm}_t \in \cT_{\qm_t}\cP} ~ \int \lagr_t \dot{\qm_t} \dx - \mathcal{L}[\qm_t, \dot{\qm}_t,t] = \sup \limits_{\dot{\qm}_t \in \cT_{\qm_t}\cP} ~ \int \lagr_t \dot{\qm_t} \dx - \cK[\qm_t, \dot{\qm}_t] + \cU[\qm_t, t] \label{eq:hamil_app}
\end{align}
Since the potential energy does not depend on $\dot{\qm}_t$, it will not affect any later derivations in this section.

We focus on \textit{kinetic energies} $\cK[\qm_t, \dot{\qm}_t]$ which are linear in the density (see \cref{def:dual_linearizable}).
We consider two primary examples, the $WFR_{\lambda}$ metric $\cK[\qm_t, \dot{\qm}_t]$ using the continuity equation with growth term dynamics, and kinetic energies defined by expectations of ground-space Lagrangian costs under $\qm_t$ (see \cref{sec:lagrangian}, \citet{villani2009optimal} Ch. 7, \cref{example:lagr_ot} below), 
\ifhl{
\begin{align}
\cK[\qm_t, \dot{\qm}
_t] &= \inf \limits_{v_t, g_t} 
\int  \Big( \underbrace{L(x, v_t)
+ \frac{\lambda}{2} g_t(x)^2 }_{\eqqcolon K(x,v_t,g_t)} \Big) \qm_t \dx, \qquad \text{s.t}
\qquad 
\dot{\qm}
_t = - \divp{ \qm_t v_t } + \lambda \qm_t g_t \label{eq:wfr-dynamics}
\end{align}
The kinetic energy above captures three important special cases:  
\begin{itemize}
    \item $W_2$ kinetic energy for $L(x, v_t) = \frac{1}{2}\|v_t(x) \|^2$ and $\lambda = 0$
    \item Optimal transport with Lagrangian costs (\cref{app:lagr})
    \item Wasserstein Fisher-Rao ($WFR_{\lambda}$) with $L(x, v_t) = \frac{1}{2}\|v_t(x) \|^2$ (\cref{app:wfr})
\end{itemize}
}
We proceed to simplify from \cref{eq:hamil_app} using the kinetic energy in \cref{eq:wfr-dynamics}.  We let $K(x,v_t,g_t) \coloneqq  L(x, v_t)
+ \frac{\lambda}{2} g_t(x)^2 $ for simplicity of notation and to call attention to the fact that these terms do not change in the derivations leading to \cref{eq:last_hamil}.
\begin{align}
\cH[\qm_t, \lagr_t, t]  &= \sup \limits_{\dot{\qm}_t \in \cT_{\qm_t}\cP} ~ \int \lagr_t \dot{\qm_t} \dx - \cK[\qm_t, \dot{\qm}_t] + \cU[\qm_t, t] \label{eq:from_continuity} \\
&
\ifhl{{=\sup \limits_{\dot{\qm}_t \in \cT_{\qm_t}\cP}  
 ~ \int \lagr_t \dot{\qm_t} \dx 
- \left( \inf \limits_{(v_t, g_t)} ~\int 
K(x, v_t, g_t) 
 ~\qm_t
\dx \right) + \cU[\qm_t, t] \st \dot{\qm}_t = - \divp{ \qm_t v_t } + \lambda \qm_t g_t }} \nonumber \\
&
=\ifhl{ \sup \limits_{\dot{\qm}_t \in \cT_{\qm_t}\cP}  \sup \limits_{(v_t, g_t)} 
\int \lagr_t \dot{\qm_t} \dx 
-  ~\int K(x, v_t, g_t) ~\qm_t \dx + \cU[\qm_t, t] \st \dot{\qm}_t = - \divp{ \qm_t v_t } + \lambda \qm_t g_t }
\nonumber \\
&
=\ifhl{ \sup \limits_{(v_t, g_t)} \sup \limits_{\dot{\qm}_t \in \cT_{\qm_t}\cP} ~ \int \lagr_t \dot{\qm_t} \dx  -  ~\int K(x, v_t, g_t) ~\qm_t \dx + \cU[\qm_t, t] \st \dot{\qm}_t = - \divp{ \qm_t v_t } + \lambda \qm_t g_t }
\nonumber
\\
&=\sup \limits_{(v_t, g_t)} ~ \int  \lagr_t \big(-\divp{ \qm_t v_t } + \lambda \qm_t g_t \big) \dx  -  ~\int K(x, v_t, g_t) ~ \qm_t \dx + \cU[\qm_t, t] \nonumber
\intertext{where, in the fourth line, we can always swap the order of the supremum and, in the last line, we use the form of the constraint to reparameterize the optimization over $\dot{\qm}_t$ in terms of $(v_t,g_t)$.}
\intertext{Integrating by parts, we have}
\cH[\qm_t, \lagr_t, t]  &= \sup \limits_{(v_t, g_t)} ~  \int \big( \langle \nabla \lagr_t,  v_t \rangle \qm_t + \lambda \qm_t  \lagr_t g_t \big) \dx -  ~\int K(x, v_t, g_t) ~\qm_t \dx + \cU[\qm_t, t]  .\label{eq:last_hamil}
\end{align}



\subsubsection{Wasserstein Fisher-Rao and $W_2$ }\label{app:wfr}

For $K(x,v_t,g_t)=  \frac{1}{2} \| v_t(x) \|^2 + \frac{\lambda}{2} g_t(x)^2$, we proceed from \cref{eq:last_hamil},
\begin{align}
\cH[\qm_t, \lagr_t, t] &=\sup \limits_{(v_t, g_t)} ~  \int \Big( \langle \nabla \lagr_t,  v_t \rangle \qm_t + \lambda \qm_t \lagr_t g_t \Big) \dx - \int  \left( \frac{1}{2} \| v_t \|^2 + \frac{\lambda}{2} g_t^2 \right) \qm_t \dx + \cU[\qm_t, t] \label{eq:blah_wfr} 
\end{align}
\normalsize
Eliminating $v_t$ and $g_t$ implies 
\begin{align}
v_t = \nabla \lagr_t \qquad \qquad g_t = \lagr_t
\end{align}
where $v_t = \nabla \lagr_t$ also holds for the $W_2$ case with $\lambda = 0$.
Substituting into \cref{eq:blah_wfr}, we obtain a Hamiltonian with a dual kinetic energy $\cK^*[\qm_t, \dot{\qm}_t]$ below that is linear in $\qm_t$ and matches the metric expressed in the cotangent space $ \frac{1}{2}\tmetricat{\lagr_t, \lagr_t}{\qm_t}^{WFR_\lambda}$,
\begin{align}
\cH[\qm_t, \lagr_t, t] &= ~ \int \Big( \frac{1}{2} \| \nabla \lagr_t\|^2 + \frac{\lambda}{2}  \lagr_t^2 \Big)\qm_t \dx + \cU[\qm_t, t] = \frac{1}{2}\tmetricat{\lagr_t, \lagr_t}{\qm_t}^{WFR_\lambda} + \cU[\qm_t, t].
\end{align}
We make a similar conclusion for the $W_2$ metric with $\lambda = 0$, where the dual kinetic energy is $ \cK^*[\qm_t, \dot{\qm}_t] =\frac{1}{2}\tmetricat{\lagr_t, \lagr_t}{\qm_t}^{W_2} = \frac{1}{2} \int   \| \nabla \lagr_t\|^2\qm_t \dx$.   Note, the kinetic energy could also use ground-space Lagrangian costs along with the FR growth term.

\subsubsection{Lifting Ground-Space Lagrangian Costs to Kinetic Energies} \label{app:lagr}
We first provide background on Lagrangian actions in the ground space and their use to define optimal transport costs, before continuing the above derivations to calculate the density-space Hamiltonian in \cref{eq:ku_lagr} - (\ref{eq:hamil_lagr}) below.

As in \citet{villani2009optimal} Thm. 7.21, we can consider using the cost associated with action-minimizing curves $\gamma^*(x_0, x_1)$ to define an optimal transport costs between densities.   We show that this corresponds to a special case of our Wasserstein Lagrangian Flows framework with kinetic energy $\cK[\qm_t, \dot{\qm}_t] = \int L(x, v_t) \qm_t \dx$ as in \cref{eq:wfr-dynamics}.   
However, as discussed in \cref{sec:wlm}, defining our Lagrangians $\cL[\qm_t, \dot{\qm}_t, t] $ \textit{directly} on the space of densities allows for more generality using kinetic energies which include growth terms or potential energies which depend on the density.

\paragraph{Lagrangian and Hamiltonian Mechanics in the Ground-Space}\label{sec:lagrangian}
We begin by reviewing action-minimizing curves in the ground space, which forms the basis the Lagrangian formulation of classical mechanics \citep{arnol2013mathematical}. 
For curves $\gamma(t): [0,1] \rightarrow \cX$ with velocity $\dot{\gamma}_t \in \cT_{\gamma(t)}\cX$, 
we consider evaluating a Lagrangian function $L(\gamma_t, \dot{\gamma}_t)$ along the curve to define the \textit{action} as the time integral $A(\gamma) = \int_0^1 L(\gamma_t, \dot{\gamma}_t) \dt$.    Given two endpoints $\x, \y \in \cX$, we consider minimizing the action along all curves  
with the appropriate endpoints
$\gamma \in \Pi(\x,\y)$,
\begin{align}\label{eq:lagr_ground_cost}
c_L(\x, \y) = \inf \limits_{\gamma \in \Pi(\x,\y) } A(\gamma) = \inf 
\limits_{\gamma_t}
\int_0^1 L(\gamma_t, \dot{\gamma}_t) \dt \st \gamma_0 = \x, \,\, \gamma_1 =\y
\end{align}
We refer to the optimizing curves $\gamma^*(\x,\y)$ as \textit{Lagrangian flows} in the ground-space, which satisfy 
the Euler-Lagrange equation $\frac{\dd}{\dt}\frac{\partial}{\partial \dot{\gamma}_t} L(\gamma_t, \dot{\gamma}_t) = \frac{\partial}{\partial\gamma_t} L(\gamma_t, \dot{\gamma}_t)$ as a stationarity condition.

We will assume that $L(\gamma_t, \dot{\gamma}_t)$ is strictly convex in the velocity $\dot{\gamma}_t$ and twice-continuously differentiable, in which case we can obtain an equivalent, \textit{Hamiltonian} perspective via convex duality. Considering momentum variables $\accel_t$, we define the Hamiltonian $\hamilgs(\gamma_t, \accel_t)$ as the Legendre transform of $L$ with respect to $\dot{\gamma}_t$,
\begin{align}\label{eq:hamiltonian_ground}
 \hamilgs(\gamma_t, \accel_t) = \sup \limits_{\dot{\gamma}_t}  ~ \langle  \dot{\gamma}_t, \accel_t \rangle -  L(\gamma_t, \dot{\gamma}_t) 
\end{align}
\normalsize
The Euler-Lagrange equations
can be written as Hamilton's equations in the phase space
\begin{align}
\dot{\gamma}_t
= \frac{\partial }{\partial \accel_t} \hamilgs(\gamma_t, \accel_t) \qquad \qquad \dot{\accel}_t
= -\frac{\partial }{\partial \gamma_t} \hamilgs(\gamma_t, \accel_t).
\end{align}
\normalsize
We proceed to consider Lagrangian actions in the ground-space as a way to construct optimal transport costs over distributions.

\begin{example}[\textbf{Ground-Space Lagrangians as OT Costs}]\label{example:lagr_ot}
The cost function $c(\x,\y)$ is a degree of freedom in specifying an optimal transport distance between probability densities $\mu_0, \mu_1 \in \cP(\cX)$ in \cref{eq:ot}.   Beyond $c(\x,\y) = \| \x - \y \|^2$, one might consider defining the \gls{OT} problem using a cost $c(\x, \y)$ induced by a Lagrangian $L(\gamma_t, \dot{\gamma}_t)$ in the ground space $\gamma_t : [0,1] \rightarrow \cX$, as in \cref{eq:lagr_ground_cost} (\citet{villani2009optimal} Ch. 7).  
In particular, a coupling $\pi(\x,\y)$ should assign mass to endpoints $(\x,\y)$ based on the Lagrangian cost of their action-minimizing curves $\gamma^*(\x,\y)$.
\ifhl{
For Tonelli Lagrangians, which respect conditions such as convexity in $v_t$ and superlinearity, we have \citep[Def. 3.3.1, Thm. 3.5.1]{schachter2017eulerian}
\begin{align}\label{eq:ground_lagr_ot_curves} 
W_{L}\left(\mu_{0},\mu_{1}\right)
&= \inf \limits_{\pi \in \Pi(\mu_0,\mu_1)} \int \int c_L(x_0,x_1) \pi(x_0,x_1) \dx_0 \dx_1   
\st 
\quad \pi(x_0) = \mu_0(x_0),
\quad \pi(x_1)= \mu_1(x_1) \\
&=\inf_{X:\left[0,1\right]\times\X\to\X}
\int_{0}^{1}\int L\left(X\left(t,x\right),\dot{X}\left(t,x\right)\right)\mu_{0}\left(x\right)\dd x\dd t\spp\st\spp X\left(1,\cd\right)_{\#}\mu_{0}=\mu_{1}
\end{align}
where $X\left(\cd,x\right)$ is a curve for every $x$.}
This optimization may also be viewed as \citep[Def. 3.4.1]{schachter2017eulerian}
\begin{align}\label{eq:ground_lagr_ot} 
 W_L(\mu_0, \mu_1)
    =&\inf  \limits_{\qm_t} \inf  \limits_{v_t} \int_0^1  \int  L(x, v_t) \qm_t \dx \dt
    ~~\,\,\text{s.t.}\,\,  \dot{\qm}_t = -\divp{ \qm_t v_t }, 
    ~ \qm_0 = \mu_0, \,  
           \qm_1 = \mu_1 . 
\end{align}
We can thus view the \gls{OT} problem as `lifting' the Lagrangian cost on the ground space $\cX$ to a distance in the space of probability densities $\wmanifold$ via the kinetic energy $\cK[\qm_t, \dot{\qm}_t] = \int L(x, v_t) \qm_t \dx$ (see below).  
Of course, the Benamou-Brenier dynamical formulation of $W_2$-\gls{OT} in \cref{eq:dynamical} may be viewed as a special case with $
L(x, v_t) = \frac{1}{2} \| v_t \|^2 $.
\end{example}

\textbf{Calculating the Hamiltonian}
Recognizing the similarity with the Benamou-Brenier formulation in \cref{example:w2ot}, we consider the Wasserstein Lagrangian optimization with two endpoint marginal constraints,
\begin{align}
\cS_\cL(\{\mu_{0,1}\}) &= \inf \limits_{\qm_t \in \marginalcouple} \int_0^1 \cK[\qm_t, \dot{\qm}_t] - \cU[\qm_t, t] \dt \label{eq:ku_lagr} \\
&= \inf \limits_{\qm_t} \int_0^1  \left( \int   L(x, v_t) \qm_t \dx - \cU[\qm_t, t] \right) \dt 
\st \quad  \qm_0 = \mu_0, \qquad \qm_1 = \mu_1 \nonumber
\end{align}
 We proceed to derive the Wasserstein Hamiltonian, where the the tangent space is represented using the continuity equation as in \cref{eq:wfr-dynamics}, \ref{eq:ground_lagr_ot}).  
 
 Continuing from \cref{eq:last_hamil} with $\lambda = 0$ (no growth dynamics),
 we have
\begin{align}
\cH[\qm_t, \lagr_t, t] &=\sup \limits_{v_t} ~  \int  \langle \nabla \lagr_t,  v_t \rangle \qm_t  \dx - \cK[\qm_t, \dot{\qm}_t] + \cU[\qm_t, t] . \nonumber \\
&= \sup \limits_{v_t} ~  \int  \langle \nabla \lagr_t,  v_t \rangle \qm_t  \dx -  \int  L(x, v_t) \qm_t \dx + \cU[\qm_t, t] \nonumber \\
&=  \int \left( \sup \limits_{v_t} ~ \langle \nabla \lagr_t,  v_t \rangle - L(x, v_t) \right)  \qm_t  \dx+ \cU[\qm_t, t]\label{eq:legendre_lagr}
\end{align}
which is simply a Legendre transform between velocity and momentum variables in the ground space (\cref{eq:hamiltonian_ground}).   We can finally write,
\begin{align}
\cH[\qm_t, \lagr_t, t] = \int \hamilgs(x, \nabla \lagr_t) \qm_t \dx + \cU[\qm_t, t] \label{eq:hamil_lagr}
\end{align}
which implies the dual kinetic energy is simply the expectation of the Hamiltonian $\cK^*[\qm_t, \lagr_t] = \int \hamilgs(x, \nabla \lagr_t) \qm_t \dx$ and is clearly linear in the density $\qm_t$.

We leave empirical exploration of various Lagrangian costs for future work, but note that $\hamilgs(x, \nabla \lagr_t)$ in \cref{eq:hamil_lagr} must be known or optimized using \cref{eq:legendre_lagr} to obtain a tractable objective.


\subsection{Schrödinger Bridge}\label{app:sb}

In this section, we derive potential energies and tractable objectives corresponding to the Schrödinger Bridge problem 
\begin{align}
\hspace*{-.2cm} S_{SB} = \inf \limits_{\qm_t, v_t}
\int_0^1 \int
 \frac{1}{2} \| v_t \|^2 \qm_t \dx \dt
 \st \dot{\qm}_t &= -\divp{ \qm_t \vel_t } - \sqrttwolagrsbat{} \Delta \qm_t
 \quad \qm_0 = \mu_0, \,\, \qm_1 = \mu_1 . \label{eq:fp_sb} 
\end{align}
which we will solve using the following (linear in $\qm_t$) dual objective from \cref{eq:sb_t_final} 
\small
\begin{align}
 \hspace*{-.2cm} \cS_{SB} = \inf \limits_{\qm_t \in \marginalcouple}\sup \limits_{\hclagr_t}  ~ \int \hclagr_1 \mu_1\dx 
  -  \int \hclagr_0 \mu_0\dx   -\int_0^1 \int  \Bigg( \frac{\partial \hclagr_t }{\partial t}  +\frac{1}{2} \big\| \nabla \hclagr_t \big\|^2  + \sqrttwolagrsbat{}  \Delta \hclagr_t
\Bigg) \qm_t \dx \dt .\nonumber
\end{align}
\normalsize
\paragraph{Lagrangian and Hamiltonian for SB}
We consider a potential energy of the form,
\begin{align}
\label{eq:potential_sb}
\cU[\qm_t, t] &= -\lagrsbat{} \int \|  \nabla \log \qm_t \|^2 \qm_t \dx    
\end{align}
which, alongside the $W_2$ kinetic energy, yields the full Lagrangian
\begin{align}
\cL[\qm_t, \dot{\qm}_t, t] =  \frac{1}{2} \wmetricat{ \dot{\qm}_t, \dot{\qm}_t }{\qm_t} + \lagrsbat{} \int \|  \nabla \log \qm_t \|^2 \qm_t \dx  . \label{eq:sb_lagr_app}
\end{align}
As in \cref{eq:from_continuity}-(\ref{eq:blah_wfr}), we parameterize the tangent space using the continuity equation $\dot{\qm}_t = -\divp{ \qm_t v_t}$ and vector field $v_t$ in solving for the Hamiltonian.   Integrating by parts, 
\begin{align}
\cH[\qm_t, \lagr_t, t] &=  \sup_{\dot{\qm}_t} \int \lagr_t \dot{\qm_t} \dx - \cL[\qm_t, \dot{\qm}_t,t] \label{eq:hamil_sb} \\
&=\sup \limits_{v_t} ~  \int  \langle \nabla \lagr_t,  v_t \rangle \qm_t   \dx - \frac{1}{2}\int \|  v_t \|^2 \qm_t \dx -\lagrsbat{t}  
       \int \|  \nabla \log \qm_t \|^2 \qm_t \dx 
    \nonumber
\end{align}
\normalsize
which implies $v_t = \nabla \lagr_t $ as before.  Substituting into the above, the Hamiltonian becomes 
\begin{align}
\cH[\qm_t, \lagr_t, t] =  \frac{1}{2}\int \| \nabla \lagr_t  \|^2 \qm_t \dx -\lagrsbat{}  
       \int \|  \nabla \log \qm_t \|^2 \qm_t \dx
     . \label{eq:gsb_hamiltonian}
\end{align}
which is of the form $\cH[\qm_t, \lagr_t, t]= \cK^*[\qm_t, \lagr_t, t] + \cU[\qm_t, t]$ and matches \citet{leger2021hopf} Eq. 8. 
As in \cref{prop:dual_obj}, the dual for the Wasserstein Lagrangian Flow with the Lagrangian in \cref{eq:sb_lagr_app} involves the Hamiltonian in \cref{eq:gsb_hamiltonian}, 
\small
\begin{align}
\cS_{\cL} =  \inf \limits_{\qm_t \in \marginalcouple} \sup \limits_{\lagr_t} \int \lagr_1  \mu_1\dx -  \int \lagr_0  \mu_0\dx 
- \int_0^1   \int \left( \frac{\partial \lagr_t}{\partial t} + \frac{1}{2} \| \nabla \lagr_t  \|^2 - \lagrsbat{} \int \|  \nabla \log \qm_t \|^2  \right) \qm_t \dx \dt \label{eq:expanded2} 
\end{align}
\normalsize
However, this objective is nonlinear in $\qm_t$ and requires access to $\nabla \log \qm_t$.   
To linearize the dual objective, we proceed using a reparameterization in terms of the Fokker-Planck equation, or using the Hopf-Cole transform, in the following proposition. 

\begin{proposition}\label{prop:sb}
The solution to the Wasserstein Lagrangian flow
\begin{align}
&\cS_{\cL}(\{\mu_{0,1}\}) = \inf
\limits_{\qm_t}
\int_0^1 \mathcal{L}[\qm_t, \dot{\qm}_t,t]\dt 
\quad \st \quad
\qm_0 = \mu_0, \qquad \qm_1 = \mu_1 \label{eq:wlf_app} \\
&\text{where}\,\,\,
 { \cK[\qm_t, \dot{\qm}_t] = \frac{1}{2} \wmetricat{ \dot{\qm}_t, \dot{\qm}_t }{\qm_t}, \qquad \cU[\qm_t, t] =  -\lagrsbat{} \wnormat{ 
    \nabla \log \qm_t
    }{\qm_t} } \nonumber
\end{align}
matches the solution to the \gls{SB} problem in \cref{eq:fp_sb}, $\cS = \cS_{SB}(\{\mu_{0,1}\}) = \cS_\cL(\{\mu_{0,1}\}) + c(\{\mu_{0,1}\})$
up to a constant $c(\{\mu_{0,1}\})$ wrt $\qm_t$. 

Further, $\cS$ is the solution to the (dual) optimization
\small
\begin{align}
\cS &=  \inf \limits_{\qm_t \in \marginalcouple} \sup \limits_{\hclagr_t}  \int \hclagr_1 \mu_1\dx  -  \int \hclagr_0 \mu_0\dx  
  -\int_0^1 \int  \Bigg( \frac{\partial \hclagr_t }{\partial t}  +\frac{1}{2} \big\| \nabla \hclagr_t \big\|^2  + \sqrttwolagrsbat{}  \Delta \hclagr_t
\Bigg) \qm_t \dx \dt . \label{eq:gsb_final}
\end{align}
\normalsize
Thus, we obtain a dual objective for the SB problem, or WLF in \cref{eq:wlf_app}, which is linear in $\qm_t$.
\end{proposition}
\begin{proof}
We consider the following reparameterization \citep{leger2021hopf}
\begin{align}
     \lagr_t = \hclagr_t - \sqrttwolagrsbat{}  \log \qm_t, \qquad \qquad \nabla \lagr_t  =  \nabla \hclagr_t- \sqrttwolagrsbat{} \nabla \log \qm_t . \label{eq:reparam_sb}
\end{align}
Note that $\lagr_t$ is the drift for the continuity equation in \cref{eq:hamil_sb}, $\dot{\qm}_t = -\divp{ \qm_t \nabla \lagr_t }$.   Via the above reparameterization, we see that $\nabla \hclagr_t $ corresponds to the drift in the Fokker-Planck dynamics 
$ \dot{\qm}_t=  -\divp{ \qm_t \nabla \hclagr_t } + \sqrttwolagrsbat{} \divp{ \qm_t \nabla \log \qm_t } =  -\divp{ \qm_t \nabla \hclagr_t } + \sqrttwolagrsbat{} \Delta \qm_t$.

\textit{Wasserstein Lagrangian Dual Objective after Reparameterization:}
Starting from the dual objective in \cref{eq:expanded2}, we perform the reparameterization in \cref{eq:reparam_sb}, $\lagr_t = \hclagr_t - \sqrttwolagrsbat{}  \log \qm_t$,
\small
\begin{align}
&\cS_\cL = \inf \limits_{\qm_t \in \marginalcouple} \sup \limits_{\hclagr_t}  \int \hclagr_1 \mu_1\dx - { \sqrttwolagrsbat{}  \int  \log \qm_1 ~ \mu_1\dx } -  \int \hclagr_0 \mu_0\dx + { \sqrttwolagrsbat{} \int  \log \qm_0 ~ \mu_0\dx}  \label{eq:cancel} \\
& -\int_0^1 \int  \Bigg( \frac{\partial \hclagr_t }{\partial t}  + \frac{\partial}{\partial t}\left(  \sqrttwolagrsbat{}\log \qm_t \right) 
+\frac{1}{2} \left\langle \nabla \hclagr_t - \sqrttwolagrsbat{} \nabla  \log \qm_t, \nabla \hclagr_t - \sqrttwolagrsbat{} \nabla  \log \qm_t \right\rangle -
\lagrsbat{} \left\| \nabla \log \qm_t\right\|^2
\Bigg) \qm_t \dx \dt
\nonumber
\end{align}
\normalsize
 Noting that the $ \int  
\sqrttwolagrsbat{} (\frac{\partial }{\partial t} \log \qm_t) ~ \qm_t \dx $ cancels since
$\frac{\dd }{\dd t} \int \qm_t \dx = 0$, we simplify to obtain
\begin{align}
\cS_\cL &=  \inf \limits_{\qm_t \in \marginalcouple} \sup \limits_{\hclagr_t}  \int \hclagr_1 \mu_1\dx - { \sqrttwolagrsbat{}  \int  \log \qm_1 ~ \mu_1\dx } -  \int \hclagr_0 \mu_0\dx + { \sqrttwolagrsbat{} \int  \log \qm_0\mu_0\dx} \nonumber  \\
  & -\int_0^1 \int  \Bigg( \frac{\partial \hclagr_t }{\partial t}  +\frac{1}{2} \big\| \nabla \hclagr_t \big\|^2  - \sqrttwolagrsbat{}  \big\langle \nabla \hclagr_t, \nabla  \log \qm_t \big\rangle 
\Bigg) \qm_t \dx \dt \nonumber
\end{align}
where the Hamiltonian now matches Eq. 7 in \citet{leger2021hopf}.
Taking $\nabla \log \qm_t = \frac{1}{\qm_t} \nabla \qm_t$ and integrating by parts, the final term becomes
\begin{align}
\cS_\cL =  \inf \limits_{\qm_t \in \marginalcouple} \sup \limits_{\hclagr_t}  &\int \hclagr_1 \mu_1\dx - { \sqrttwolagrsbat{}  \int  \log \qm_1 ~ \mu_1\dx } -  \int \hclagr_0 \mu_0\dx + { \sqrttwolagrsbat{} \int  \log \qm_0\mu_0\dx} \nonumber  \\
  & -\int_0^1 \int  \Bigg( \frac{\partial \hclagr_t }{\partial t}  +\frac{1}{2} \big\| \nabla \hclagr_t \big\|^2  + \sqrttwolagrsbat{} \Delta \hclagr_t
\Bigg) \qm_t \dx \dt \nonumber
\end{align}
Finally, we add terms  $c(\{\mu_{0,1}\}) = \sqrttwolagrsbat{}  \int  (\log \mu_1) ~ \mu_1\dx -\sqrttwolagrsbat{}  \int ( \log \mu_0) ~ \mu_0\dx $ which are constant with respect to $\qm_{0,1}$,
\begin{align}
\cS_\cL(\{\mu_{0,1}\}) + c(\{\mu_{0,1}\}) &=  \inf \limits_{\qm_t \in \marginalcouple} \sup \limits_{\hclagr_t}  \int \hclagr_1 \mu_1\dx + { \sqrttwolagrsbat{1}  \int \left( \log \mu_1 ~ - \log \qm_1 \right) \mu_1\dx } \label{eq:wlf_dual_sb} \\
&\phantom{ \inf \limits_{\qm_t \in \marginalcouple} \sup \limits_{\hclagr_t} } -  \int \hclagr_0 \mu_0\dx -  \sqrttwolagrsbat{0} \int  \left( \log \mu_0 ~ - \log \qm_0\right)\mu_0 \dx  \nonumber \\
  &\phantom{ \inf \limits_{\qm_t \in \marginalcouple} \sup \limits_{\hclagr_t} } -\int_0^1 \int  \Bigg( \frac{\partial \hclagr_t }{\partial t}  +\frac{1}{2} \big\| \nabla \hclagr_t \big\|^2  + \sqrttwolagrsbat{t}  \Delta \hclagr_t
\Bigg) \qm_t \dx \dt \nonumber
\end{align}
The endpoint terms now vanish for $\qm_t \in \marginalcouple$ satisfying the endpoint constraints,
\begin{align}\label{eq:wlf_dual_sb2}
&\cS_\cL(\{\mu_{0,1}\}) + c(\{\mu_{0,1}\})  \\
&=  \inf \limits_{\qm_t \in \marginalcouple} \sup \limits_{\hclagr_t}  \int \hclagr_1 \mu_1\dx  -  \int \hclagr_0 \mu_0\dx -\int_0^1 \int  \Bigg( \frac{\partial \hclagr_t }{\partial t}  +\frac{1}{2} \big\| \nabla \hclagr_t \big\|^2  + \sqrttwolagrsbat{}  \Delta \hclagr_t
\Bigg) \qm_t \dx \dt , \nonumber
\end{align}
which matches the dual in \cref{eq:gsb_final}.   We now show that this is also the dual for the \gls{SB} problem.




\textit{Schrödinger Bridge Dual Objective:}
Consider the optimization in \cref{eq:fp_sb} (here, $t$ may be time-dependent)
\begin{align}
\cS_{SB}(\{\mu_{0,1}\}) = \inf \limits_{\qm_t, v_t}
\int_0^1 
  \int \frac{1}{2} \| v_t \|^2 \qm_t \dx \dt
 \st \dot{\qm}_t  &= -\divp{ \qm_t \vel_t } + \sqrttwolagrsbat{t} \divp{ \qm_t ~\nabla \log \qm_t },\,  \qm_0 = \mu_0, \,\, \qm_1 = \mu_1  
\label{eq:sb_app}
\end{align}
\normalsize
We treat the optimization over $\qm_t$ as an optimization over a vector space of functions, which is later constrained be normalized via the $\qm_0 = \mu_0, \qm_1 = \mu_1$ constraints and continuity equation (which preserves normalization).   It is also constrained to be nonnegative, but we omit explicit constraints for simplicity of notation.  The optimization over $v_t$ is also over a vector space of functions.  See \cref{app:equivalence} for additional discussion.

Given these considerations, we may now introduce Lagrange multipliers $\lambda_0, \lambda_1$ to enforce the endpoint constraints and $\hclagr_t$ to enforce the dynamics constraint, 
\small 
\begin{align}
\cS_{SB}(\{\mu_{0,1}\}) &= \inf \limits_{\qm_t, v_t} \sup\limits_{\hclagr_t, \lambda_{0,1}}
\int_0^1 \int \frac{1}{2} \| v_t \|^2 \qm_t \dx \dt
 + \int_0^1 \int  \hclagr_t \left( \dot{\qm}_t + \divp{ \qm_t \vel_t } - \sqrttwolagrsbat{t} \divp{ \qm_t ~\nabla \log \qm_t }\right) \dx \dt \label{eq:gsb11} \\
 &\phantom{=\inf \limits_{\qm_t, v_t} \sup\limits_{\hclagr_t, \lambda_{0,1}}} + \int \lambda_1 \left(\qm_1 - \mu_1  \right)\dx + \int \lambda_0 \left(\qm_0 - \mu_0 \right)\dx \nonumber \\
&=  \inf \limits_{\qm_t, v_t}\sup\limits_{\hclagr_t, \lambda_{0,1}}
\int_0^1 \int \frac{1}{2} \| v_t \|^2 \qm_t \dx dt + \int \hclagr_1 \qm_1\dx - \int \hclagr_0 \qm_0\dx
 - \int_0^1 \int \frac{\partial \hclagr_t}{\partial t} \qm_t \dx \dt \\
 &\phantom{====} - \int_0^1 \int  \left\langle \nabla \hclagr_t, \vel_t - \sqrttwolagrsbat{t} \nabla \log \qm_t \right \rangle \qm_t \dx \dt 
 + \int \lambda_1 \left(\qm_1 - \mu_1  \right)\dx + \int \lambda_0 \left(\qm_0 - \mu_0 \right)\dx \nonumber
\end{align}
\normalsize
Note that we can freely we can swap the order of the optimizations since the \gls{SB} optimization in \cref{eq:sb_app} is convex in $\qm_t, v_t$, while the dual optimization is linear in $\hclagr_t, \lambda$.

Swapping the order of the optimizations and eliminating $\qm_0$ and $\qm_1$ implies $\lambda_1 = \hclagr_1$ and $\lambda_0 = \hclagr_0$, while eliminating $v_t$ implies $v_t = \nabla \hclagr_t$.  Finally, we obtain
\small 
\begin{align}
\cS_{SB}(\{\mu_{0,1}\}) &= 
 \sup\limits_{\hclagr_t} \inf \limits_{\qm_t} 
\int \hclagr_1 \mu_1\dx - \int \hclagr_0 \mu_0\dx - \int_0^1 \int \left( \frac{\partial \hclagr_t}{\partial t} + \frac{1}{2} \| \hclagr_t \|^2 
- \sqrttwolagrsbat{t} \left\langle \nabla \hclagr_t, \nabla \log \qm_t \right\rangle
\right) \qm_t \dx \dt \nonumber \\
&= \inf \limits_{\qm_t}\sup\limits_{\hclagr_t}
\int \hclagr_1 \mu_1\dx - \int \hclagr_0 \mu_0\dx - \int_0^1 \int \left( \frac{\partial \hclagr_t}{\partial t} + \frac{1}{2} \| \hclagr_t \|^2 
+ \sqrttwolagrsbat{t} \Delta \hclagr_t
\right) \qm_t \dx \dt
\label{eq:dual_gsb22}
\end{align}
\normalsize
where we swap the order of optimization again in the second line.
This matches the dual in \cref{eq:wlf_dual_sb} for $\cS_\cL(\{\mu_{0,1}\}) + c(\{\mu_{0,1}\})$ if $\sqrttwolagrsbat{}$ is independent of time, albeit without the endpoint constraints.   However, we have shown above that the optimal $\lambda_0^* = \hclagr_0^*$, $\lambda_1^* = \hclagr_1^*$ will indeed enforce the endpoint constraints.    This is the desired result in \cref{prop:sb}. 
\end{proof}

\begin{example}[\textbf{Schrödinger Bridge with Time-Dependent Diffusion Coefficient}]\label{example:sb_app}
To incorporate a time-dependent diffusion coefficient for the classical \gls{SB} problem, we modify the potential energy with an additional term 
\begin{align}
\label{eq:potential_sb2}
\cU[\qm_t, t] &= -\lagrsbat{t} \int \|  \nabla \log \qm_t \|^2 \qm_t \dx    
+  \int \left( \frac{\partial }{\partial t}\sqrttwolagrsbat{t}\right)  \log \qm_t~ \qm_t \dx 
\end{align}
This potential energy term is chosen carefully to cancel with the term appearing after reparameterization using $ \lagr_t = \hclagr_t - \sqrttwolagrsbat{t}  \log \qm_t$ in \cref{eq:cancel}.  In this case, 
\small
\begin{align*}
\int \frac{\partial \lagr_t}{\partial t}\qm_t \dx &= \int   \left( \frac{\partial \hclagr_t}{\partial t} - \frac{\partial }{\partial t}\Big(\sqrttwolagrsbat{t} \log \qm_t \Big)\right)\qm_t \dx \\
&=\int   \left( \frac{\partial \hclagr_t}{\partial t} -
 \Big(\frac{\partial }{\partial t}\sqrttwolagrsbat{t} \Big) \log \qm_t -    
 \sqrttwolagrsbat{t} \Big(\frac{\partial }{\partial t} \log \qm_t \Big) \right) ~ \qm_t \dx \\
 &= \int    \left( \frac{\partial \hclagr_t}{\partial t} -
 \Big(\frac{\partial }{\partial t}\sqrttwolagrsbat{t} \Big) \log \qm_t  \right) ~ \qm_t \dx
\end{align*}
\normalsize
where the score term cancels as before.   The additional potential energy term is chosen to cancel the remaining term.   All other derivations proceed as above, which yields an identical dual objective 
\begin{align}
 \hspace*{-.2cm} \cS_{SB} =  \inf \limits_{\qm_t \in \marginalcouple} \sup \limits_{\hclagr_t} ~ \int \hclagr_1 \mu_1 \dx
  -  \int \hclagr_0 \mu_0 \dx  -\int_0^1 \int  \Bigg( \frac{\partial \hclagr_t }{\partial t}  +\frac{1}{2} \big\| \nabla \hclagr_t \big\|^2  + \sqrttwolagrsbat{t}  \Delta \hclagr_t
\Bigg) \qm_t \dx \dt \nonumber
\end{align}
\end{example}

\subsection{Schrödinger Equation}\label{app:other_examples}

\begin{example}[\textbf{Schrödinger Equation}]\label{example:se}
Intriguingly, we obtain the Schrödinger Equation via a simple change of sign in the potential energy
$\cU[\qm_t, t] = \lagrsbat{t} \int \|\nabla  \log \qm_t \|^2 \qm_t \dx$ compared to \cref{eq:potential_sb} or, in other words, an \text{imaginary} weighting $i \sigma_t$ of the gradient norm of the Shannon entropy,
\begin{align}
\cL[\qm_t, \dot{\qm}_t, t] =  \frac{1}{2} \wmetricat{ \dot{\qm}_t, \dot{\qm}_t }{\qm_t} - \int \left[\frac{1}{8} \| \nabla \log \qm_t \|^2  + V_t(x) \right] \qm_t \dx
\end{align}
This Lagrangian corresponds to a Hamiltonian $\cH[\qm_t, \lagr_t, t] = \frac{1}{2} \wcometricat{ \lagr_t, \lagr_t }{\qm_t} + \int \left[\frac{1}{8} \| \nabla \log \qm_t \|^2 + V_t(x)\right]  \qm_t \dx $,
which leads to the dual objective
\begin{align}
\begin{split}
\cS_{SE} =~& \sup \limits_{\lagr_t} \inf \limits_{\qm_t} ~ \int \lagr_1 \mu_1 \dx-  \int \lagr_0 \mu_0 \dx -\int_0^1 \int  \left( \frac{\partial \lagr_t }{\partial t}  + \frac{1}{2} \| \nabla \lagr_t \|^2  + \frac{1}{8} \left\| \nabla \log \qm_t\right\|^2 + V_t(x)  \right) \qm_t \dx \dt. \label{appeq:dual_SE}
\end{split}
\end{align}
Unlike the Schrödinger Bridge problem, the Hopf-Cole transform does not linearize the dual objective in density.
Thus, we cannot approximate the dual using only the Monte Carlo estimate.

The first-order optimality conditions for \cref{appeq:dual_SE} are
\begin{align}
    \dot{\qm_t} = -\divp{\qm_t\nabla\lagr_t}, \;\; \frac{\partial \lagr_t }{\partial t}  + \frac{1}{2} \| \nabla \lagr_t \|^2 = \frac{1}{8} \left\| \nabla \log \qm_t\right\|^2 + \frac{1}{4} \Delta \log \qm_t - V_t(x) \label{appeq:polar_SE}
\end{align}
Note, that \cref{appeq:polar_SE} is the Madelung transform of the Schrödinger equation, i.e. for the equation
\begin{align}
    \deriv{}{t}\psi_t(x) = -i \hat{H}\psi_t(x), \;\;\text{ where }\;\; \hat{H} = -\frac{1}{2}\Delta + V_t(x),
\end{align}
the wave function $\psi_t(x)$ can be written in terms $\psi_t(x) = \sqrt{\qm_t(x)}\exp(i\lagr_t(x))$.
Then the real and imaginary part of the Schrödinger equation yield \cref{appeq:polar_SE}.
\end{example}

\section{Lagrange Multiplier Approach}\label{app:equivalence}

Our \cref{prop:dual_obj} is framed completely in the abstract space of densities and the Legendre transform between functionals of $\dot{\qm}_t \in \cT_{\qm_t}\cP$ and $\lagr_{\dot{\qm}_t } \in \cT_{\qm_t}^*\cP$.   We contrast this approach with optimizations such as the Benamou-Brenier formulation in \cref{eq:dynamical}, which are formulated in terms of the state space dynamics such as the continuity equation $\dot{\qm}_t  = -\nabla \cdot(\qm_t v_t)$.   
In this appendix, we claim that the latter approaches require a potential energy $\cU[\qm_t,t]$ which is concave or linear in $\qm_t$.
We restrict attention to continuity equation dynamics in this section, although similar reasoning holds with growth terms.

In particular, consider optimizing $\qm_t, v_t$ over a topological vector space of functions.   The notable difference here is that $\qm_t: \cX \rightarrow \bbR$ is a function, which we later constrain to be a normalized probability density using $\qm_0 = \mu_0, \qm_1 = \mu_1$, the continuity equation $\dot{\qm}_t  = -\nabla \cdot(\qm_t v_t)$ (which preserves normalization), and nonnegativity constraints.  Omitting the latter for simplicity of notation, we consider a (Tonelli) Lagrangian kinetic energy, with an arbitrary potential energy (that is concave in $\qm_t$),
\begin{align}
\cS = \inf \limits_{\qm_t, v_t} \int_0^1 \int L(x, v_t) \qm_t \dx \dt - \int_0^1 \cU[\qm_t,t]\dt \st \dot{\qm}_t  = -\nabla \cdot(\qm_t v_t) \quad \qm_0 = \mu_0, \,\, \qm_1=\mu_1
\end{align}
Since we are now optimizing $\qm_t$ over a vector space, we introduce Lagrange multipliers $\lambda_{0,1}$ to enforce the endpoint constraints and $\lagr_t$ to enforce the continuity equation.  Integrating by parts in $t$ and $x$, we have
\small
\begin{align}
\cS &= \inf \limits_{\qm_t, v_t} \sup \limits_{\lambda_{0,1}, \lagr_t} \int_0^1 \int L(x, v_t) \qm_t \dx \dt - \int_0^1 \cU[\qm_t,t]\dt + \int_0^1 \int \lagr_t \dot{\qm}_t \dx\dt + \int_0^1 \int  \lagr_t ~ \nabla \cdot(\qm_t v_t) \dx \dt \nonumber \\
&\phantom{\inf \limits_{\qm_t, v_t} \sup \limits_{\lambda_{0,1}, \lagr_t} } + \int \lambda_0 (\qm_0 - \mu_0)  \dx  + \int \lambda_1 (\qm_1 - \mu_1)  \dx  \\
& = \inf \limits_{\qm_t, v_t} \sup \limits_{\lambda_{0,1}, \lagr_t} \int_0^1 \int L(x, v_t) \qm_t \dx \dt - \int_0^1 \cU[\qm_t,t]\dt - \int_0^1 \int \frac{\partial \lagr_t }{\partial t} {\qm}_t \dx\dt - \int_0^1 \int \langle \nabla \lagr_t , v_t\rangle \qm_t \dx \dt \nonumber \\
&\phantom{\inf \limits_{\qm_t, v_t} \sup \limits_{\lambda_{0,1}, \lagr_t} } + 
\int \lagr_1 \qm_1 \dx - \int \lagr_0 \qm_0 \dx +  
\int \lambda_0 \qm_0 \dx - \int \lambda_0 \mu_0\dx  + \int \lambda_1 \qm_1 \dx -  \int \lambda_1 \mu_1\dx  \label{eq:need_linear_h}
\end{align}
\normalsize
To make further progress by swapping the order of the optimizations, we require that \cref{eq:need_linear_h} is convex in $\qm_t, v_t$ and concave in $\lambda_{0,1}, \lagr_t$.   
However, to facilitate this, we require that $\cU[\qm_t,t]$ is concave in $\qm_t$, which is an additional constraint which was not necessary in the proof of \cref{prop:dual_obj}.

By swapping the order of optimization to eliminate $\rho_0, \rho_1$ and $v_t$, we obtain the optimality conditions 
\begin{align}
\lambda_0 = \lagr_0, \,\, \lambda_1 = -\lagr_1 \qquad \qquad v_t = \nabla_{p} H (x, \nabla \lagr_t)
\end{align}
where the gradient is with respect to the second argument.   Swapping the order of optimizations again, the dual becomes
\begin{align}
\cS & =   \inf \limits_{\qm_t} \sup \limits_{\lagr_t} \int \lagr_1 \mu_1\dx - \int \lagr_0 \mu_0 \dx -  \int_0^1 \left( \int \Big( \frac{\partial \lagr_t}{\partial t} + H(x, \nabla \lagr_t)\Big) \qm_t \dx + \cU[\qm_t,t] \right) \dt . \nonumber
\end{align}
which is analogous to \cref{eq:dual_intermediate} in \cref{prop:dual_obj} for the Lagrangian kinetic energy.
While the dual above does not explicitly enforce the endpoint marginals on $\qm_t$,  the conditions $\lambda^*_0 = \lagr^*_0, \,\, \lambda_1^* = \lagr_1^*$ serve to enforce the constraint at optimality.
\section{Expressivity of Parameterization}\label{app:expressivity}

\express*
\begin{proof}
For every absolutely-continuous distributional path $\qm_t$, we have a unique gradient flow $\nabla s_t^*(x)$ satisfying the continuity equation (\citet{ambrosio2008gradient} Thm. 8.3.1),
\begin{align}
    \dot{\qm_t} = -\divp{\qm_t\nabla s_t^*(x)}\,. \label{eq:cont_eq}
\end{align}
Consider the function
\begin{align}
    \varphi_t(x_0, x_1) = 
    \begin{cases}
    x_0 + \int_0^t \nabla s_\tau^*(x_\tau) \dd\tau, \;\; t \leq 1/2\,, \\ 
    x_1 + \int_1^t \nabla s_\tau^*(x_\tau) \dd\tau, \;\; t > 1/2\,, \\ 
    \end{cases}
\end{align}
which integrates the ODE $\frac{\dx}{\dt} = \nabla s_t^*(x)$ forward starting from $x_0$ for $t \leq 1/2$, and integrates the same ODE backwards starting from $x_1$ otherwise.

Clearly, for $t \leq 1/2$ the designed function serves as a push-forward map for the samples $x_0 \sim \qm_0$, and produces samples from $\qm_t$ by \cref{eq:cont_eq}.
The same applies for $t > 1/2$.
Thus, $\varphi_t$ samples from the correct marginals, i.e.
\begin{align}
    \int\int \delta(x_t-\varphi_t(x_0, x_1)) \qm_0(x_0) \qm_1(x_1)\dx_0\dx_1 = \qm_t(x), \;\; \forall t \in [0,1].
\end{align}

We now show that $\varphi_t(x_0, x_1)$ can be expressed using the parameterization in \cref{eq:nn_two}, which constructs $x_t$ as
\begin{align}
    x_t = (1-t)x_0 + tx_1 + t(1-t)\textsc{nnet}^*(t,x_0,x_1,
    \mathbbm{1}[t < 0.5]; 
    \eta), \;\; x_0 \sim \mu_0, \;\; x_1 \sim \mu_1.
\end{align}
Then taking the function $\textsc{nnet}^*(t,x_0,x_1, \mathbbm{1}[t < 0.5]; \eta)$ as follows
\begin{align}
    \textsc{nnet}^*(t, x_0, x_1, \mathbbm{1}[t < 0.5]; \eta) = 
    \begin{cases}
    \frac{1}{1-t}\left(x_0-x_1 + \frac{1}{t}\int_0^t \nabla s_\tau^*(x_\tau) \dd\tau \right), \;\; t \leq 1/2\,, \\ 
    \frac{1}{t}\left(x_1-x_0 + \frac{1}{1-t}\int_1^t \nabla s_\tau^*(x_\tau) \dd\tau \right), \;\; t > 1/2\,, \\ 
    \end{cases}
\end{align}
we have 
\begin{align}
    (1-t)x_0 + tx_1 + t(1-t)\textsc{nnet}^*(t,x_0,x_1; \mathbbm{1}[t < 0.5]; \eta) = \varphi_t(x_0, x_1),
\end{align}
which samples from the correct marginals by construction.
\end{proof}

\section{Details of Experiments}
\label{app:experiments_details}

\subsection{Single-cell Experiments}
\label{app:sc_experiments_details}



We consider low dimensional (\cref{tab:5dim}) and high dimensional (\cref{tab:manydim}) single-cell experiments following the 
setups in \citet{tong2023improving,tong2023simulation}. The Embryoid body (\textbf{EB}) dataset \citet{moon2019visualizing} and the CITE-seq (\textbf{Cite}) and Multiome (\textbf{Multi}) datasets \citep{burkhardt2022multimodal} are repurposed and preprocessed by \citet{tong2023improving,tong2023simulation} for the task of trajectory inference. 

The \textbf{EB} dataset is a scRNA-seq dataset of human embryonic stem cells used to observe differentiation of cell lineages \citep{moon2019visualizing}. It contains approximately 16,000 cells (examples) after filtering, of which the first 100 principle components over the feature space (gene space) are used. For the low dimensional (5-dim) experiments, we consider only the first 5 principle components. The \textbf{EB} dataset comprises a collection of 5 timepoints sampled over a period of 30 days. 

The \textbf{Cite} and \textbf{Multi} datasets are taken from the Multimodal Single-cell Integration challenge at NeurIPS 2022 \citep{burkhardt2022multimodal}. Both datasets contain single-cell measurements from CD4+ hematopoietic
stem and progenitor cells (HSPCs) for 1000 highly variables genes and over 4 timepoints collected on days 2, 3, 4, and 7. 
We use the \textbf{Cite} and \textbf{Multi} datasets for both low dimensional (5-dim) and high dimensional (50-dim, 100-dim) experiments. We use 100 computed principle components for the 100-dim experiments, then select the first 50 and first 5 principle components for the 50-dim and 5-dim experiments, respectively. 
Further details regarding the raw dataset can be found at the competition website. \footnote{\url{https://www.kaggle.com/competitions/open-problems-multimodal/data}}

For all experiments, we train $k$ independent models over $k$ partitions of the single-cell datasets. The training data partition is determined by a left out intermediary timepoint. We then average test performance over the $k$ independent model predictions computed on the respective left-out marginals. For experiments using the \textbf{EB} dataset, we train 3 independent models using marginals from timepoint partitions $[1, 3, 4, 5], [1, 2, 4, 5], [1, 2, 3, 5]$ and evaluate each model using the respective left-out marginals at timepoints $[2], [3], [4]$. Likewise, for experiments using \textbf{Cite} and \textbf{Multi} datasets, we train 2 independent models using marginals from timepoint partitions $[2, 4, 7], [2, 3, 7]$ and evaluate each model using the respective left-out marginals at timepoints $[3], [4]$.

For both $\lagr_t(x,\theta)$ and $\qm_t(x,\eta)$, we consider Multi-Layer Perceptron (MLP) architectures and a common optimizer \citep{loshchilov2017decoupled}.
For detailed description of the architectures and hyperparameters we refer the reader to the code supplemented.

\subsection{Single-step Image Generation via Optimal Transport}
\label{app:single_step_image_gen}

Learning the vector field that corresponds to the optimal transport map between some prior distribution (e.g. Gaussian) and the target data allows to generate data samples evaluating the vector field only once.
Indeed, the optimality condition (Hamilton-Jacobi equation) for the dynamical optimal transport yields
\begin{align}
   \ddot{X}_t &= \nabla \left[ \frac{\partial  \lagr_t(x_t)}{\partial t} + \frac{1}{2} \| \nabla \lagr_t(x_t) \|^2 \right] = 0\,,
\end{align}
hence, the acceleration along every trajectory is zero.
This implies that the learned vector field can be trivially integrated, i.e.
\begin{align}
    X_1 = X_0 + \nabla \lagr_0(X_0)\,.
\end{align}
Thus, $X_1$ is generated with a single evaluation of $\nabla \lagr_0(\cdot)$.

For the image generation experiments, we follow common practices of training the diffusion models \citep{song2020score}, i.e. the vector field model $\lagr_t(x,\theta)$ uses the U-net architecture \citep{ronneberger2015u} with the time embedding and hyperparameters from \citep{song2020score}.
For the distribution path model $\qm_t(x,\eta)$, we found that the U-net architectures works best as well.
For detailed description of the architectures and hyperparameters we refer the reader to the code supplemented.

\begin{figure}[h]
    \centering
    \includegraphics[width=0.925\textwidth]{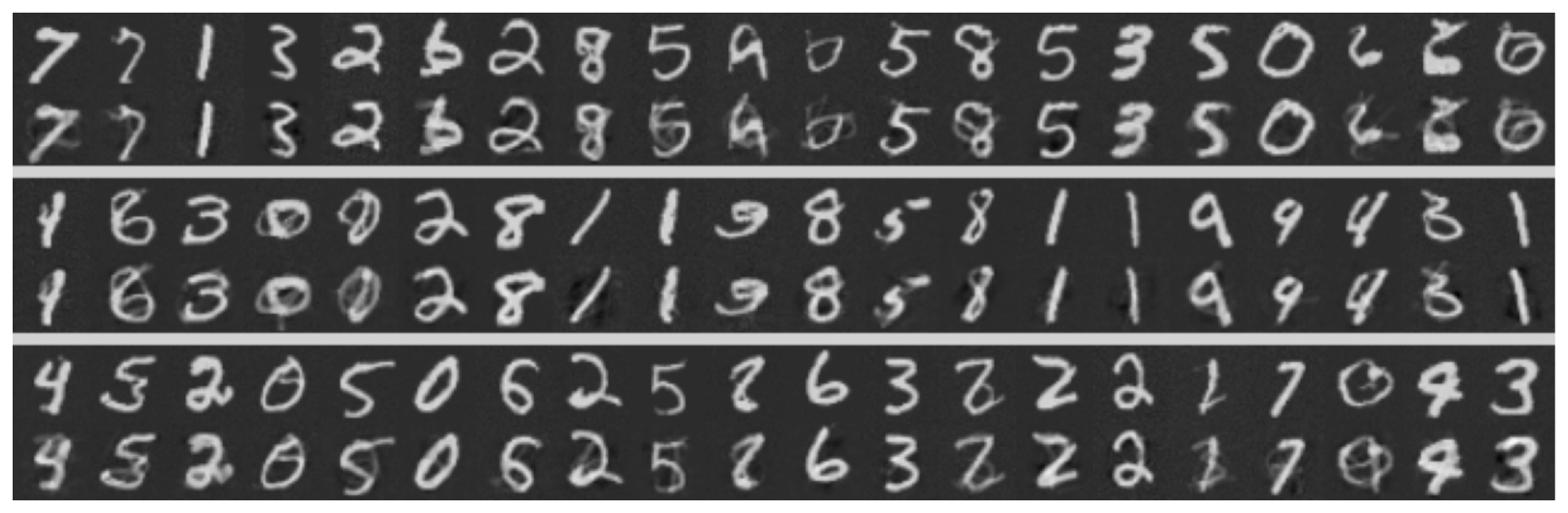}
    \includegraphics[width=0.925\textwidth]{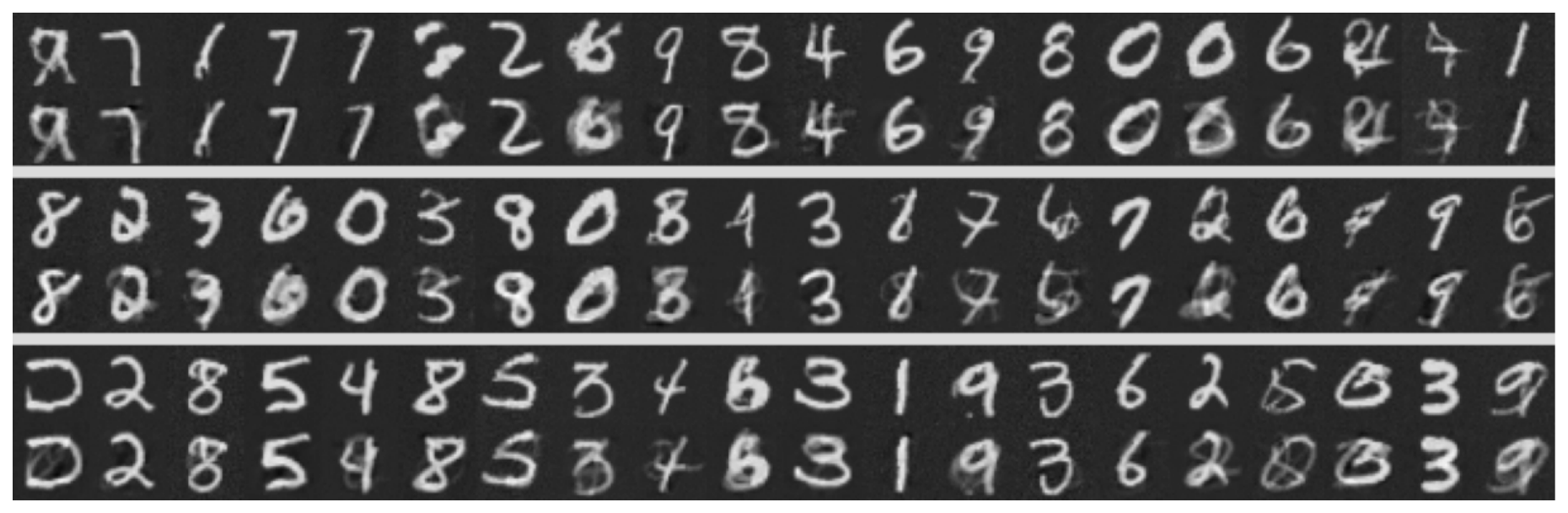}
    \includegraphics[width=0.925\textwidth]{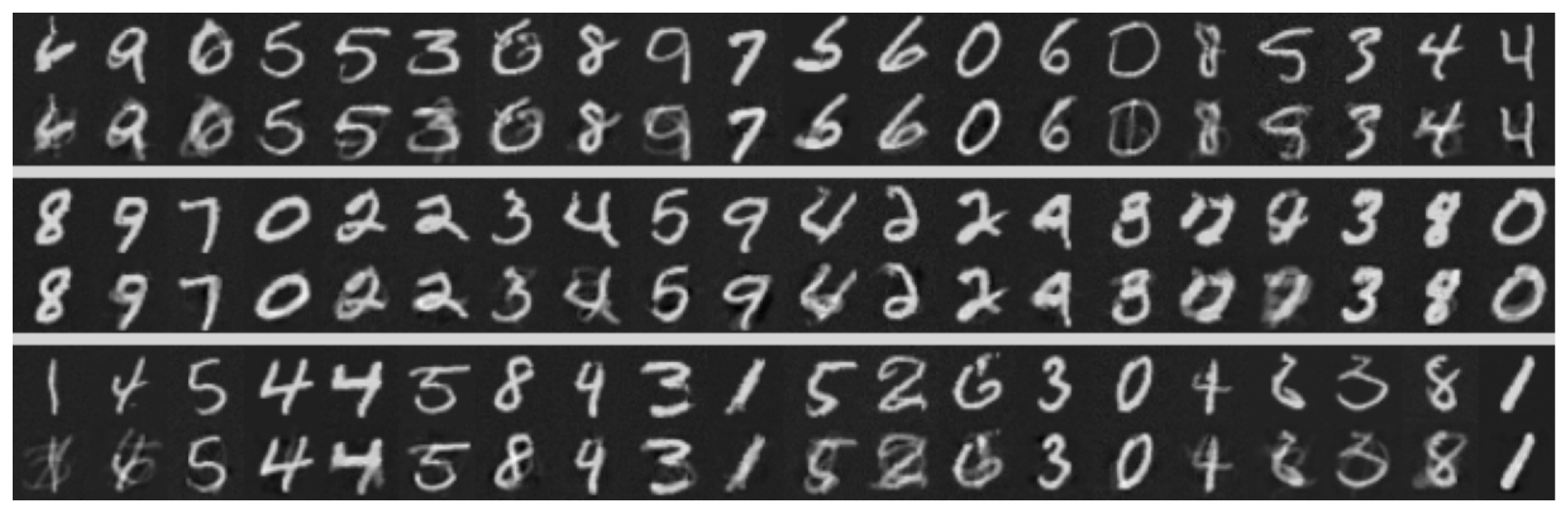}
    \caption{MNIST 32x32 image generation. Every top row is the integration of the corresponding ODE via Dormand-Prince's $5/4$ method, which makes $108$ function evaluations. Every bottom row corresponds to single function evaluation approximation.}
    \label{fig:mnist_ot}
\end{figure}
\begin{figure}
    \centering
    \includegraphics[width=0.95\textwidth]{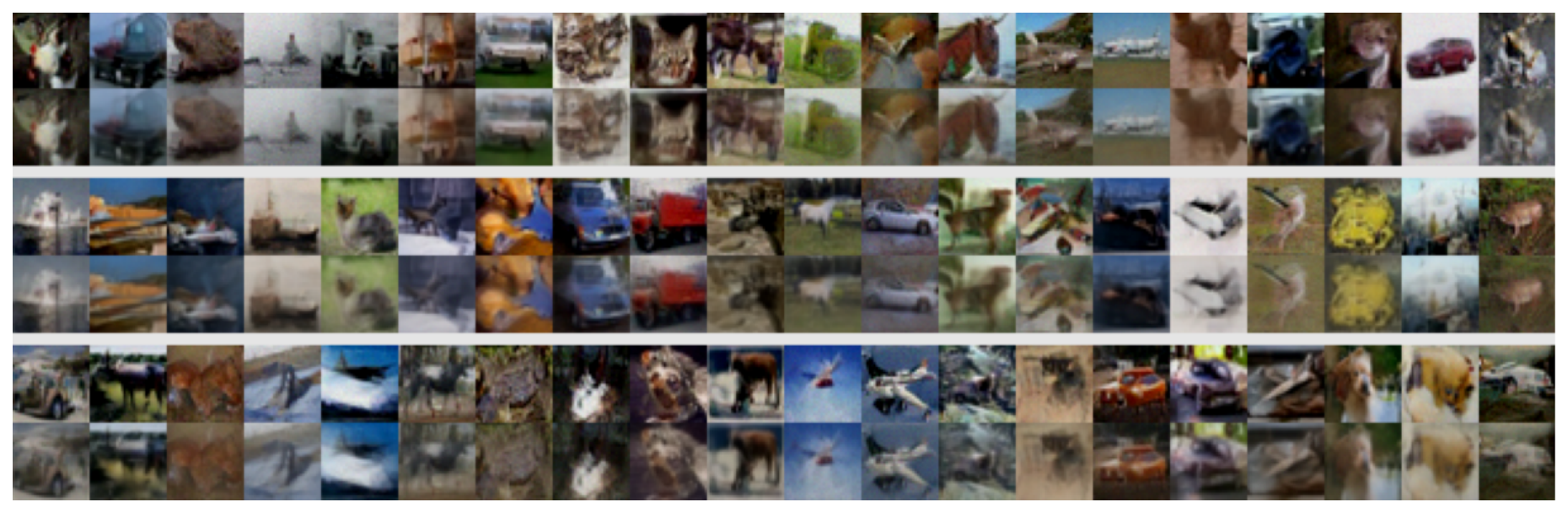}
    \includegraphics[width=0.95\textwidth]{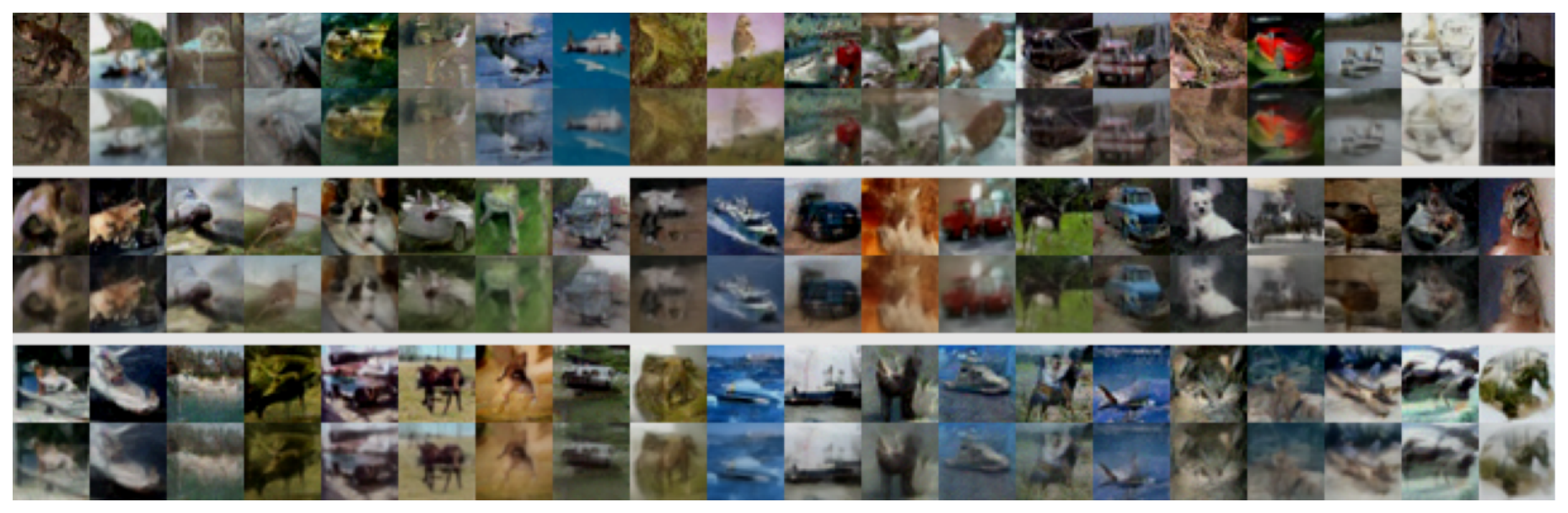}
    \includegraphics[width=0.95\textwidth]{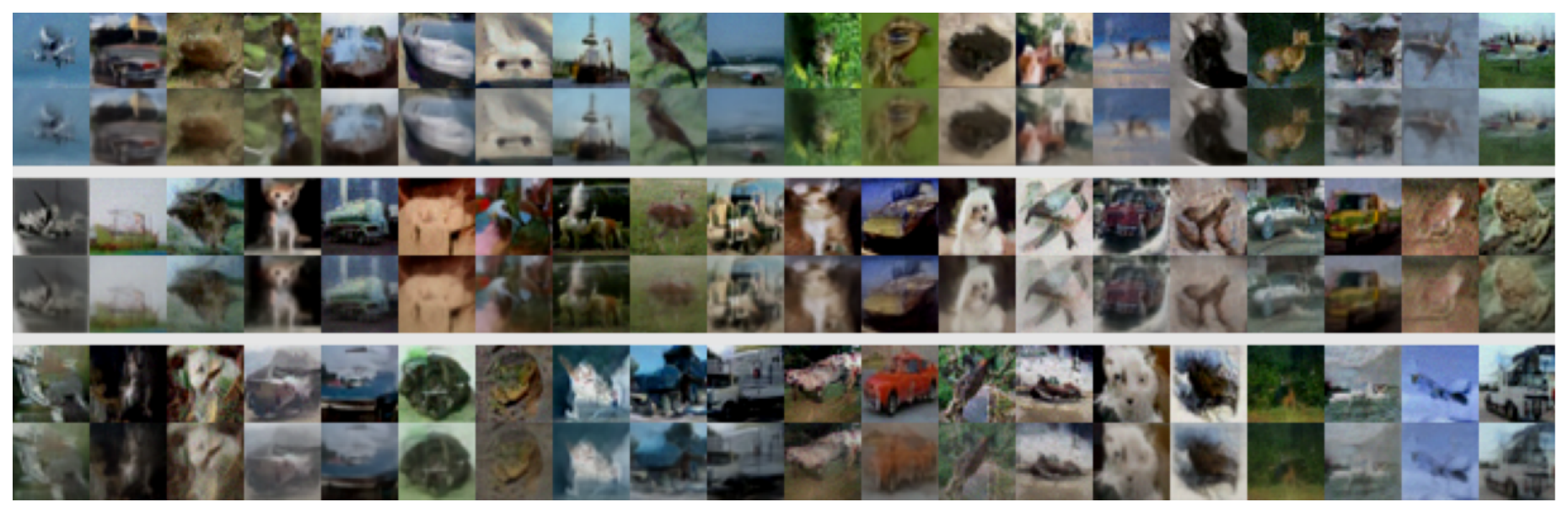}
    \caption{CIFAR-10 image generation. Every top row is the integration of the corresponding ODE via Dormand-Prince's $5/4$ method, which makes $78$ function evaluations. Every bottom row corresponds to single function evaluation approximation.}
    \label{fig:cifar10_ot}
\end{figure}

\end{document}